\documentclass{article}

\PassOptionsToPackage{numbers, compress}{natbib}
\usepackage[main, final]{neurips_2026}

\usepackage[utf8]{inputenc}
\usepackage[T1]{fontenc}
\usepackage{url}
\usepackage{booktabs}
\usepackage{amsfonts}
\usepackage{nicefrac}
\usepackage[table]{xcolor}

\usepackage[colorlinks,
            linkcolor=red,
            anchorcolor=red,
            citecolor=brown
            ]{hyperref}
\usepackage{IEEEtrantools}
\usepackage{amsthm}
\newtheorem{theorem}{Theorem}

\usepackage{thmtools}
\usepackage{thm-restate}
\usepackage{enumitem}
\usepackage{multirow}
\usepackage{amssymb}
\usepackage{algorithm}
\usepackage{algorithmic}
\usepackage{graphicx} 
\usepackage{wrapfig}
\usepackage{caption}
\usepackage{placeins}
\usepackage[absolute,overlay]{textpos}
\usepackage{subcaption}
\usepackage{times}
\usepackage{array}
\usepackage{chemformula}
\usepackage{nicematrix}
\usepackage{pgfgantt}
\newtheorem{definition}{Definition}

\definecolor{byebyebabeblue}{RGB}{31, 119, 180}
\definecolor{lovelyred}{RGB}{235, 87, 112}
\definecolor{tinyred}{RGB}{245, 185, 185}
\definecolor{gray}{RGB}{118, 113, 113}
\definecolor{lightgray}{RGB}{173, 168, 168}
\definecolor{tableblue}{RGB}{205, 232, 248}

\title{Horizon Adaptive Offline Policy Learning via\\ Value Stitching}

\author{%
  Kexin Zheng$^{1}$\quad\quad
  Xianyuan Zhan$^{2}$\footnotemark[2]~~\quad\quad
  Xintao Yan$^{1}$\footnotemark[2]\\
$^1$ The University of Hong Kong \quad$^2$ Tsinghua University\\
\texttt{whiterr-zheng@connect.hku.hk}\\ 
\texttt{zhanxianyuan@air.tsinghua.edu.cn}\quad \texttt{xintaoy@hku.hk}
}

\begin{document}
\maketitle
\renewcommand{\thefootnote}{\fnsymbol{footnote}}
\footnotetext[2]{Corresponding author.}

\begin{abstract}
Learning accurate value functions plays a decisive role for reinforcement learning (RL) agents to solve long-horizon, complex tasks. Conventional temporal-difference (TD) learning objectives suffer from value-estimation bias that accumulates over the horizon, while extended-horizon modeling methods, such as $n$-step TD backups and Q-chunking, adopt a rigid, fixed-horizon value-modeling recipe that is often not flexible enough to capture complex value structures in long-horizon, multi-stage tasks. In this paper, we show that enabling value updates with dynamic horizon composition can yield a strong offline policy learning scheme. Our method, \textit{Horizon Adaptive Offline Policy Learning via \underline{VA}lue \underline{ST}itching} (\textcolor{byebyebabeblue}{\textbf{\textit{VAST}}}), replaces fixed-horizon backups with recursive, horizon-adaptive value composition. Its key ingredient is to couple value optimization with a future state- and horizon-length-conditioned \textbf{\emph{auxiliary value function}} that is learned through direct data supervision, and a \textbf{\emph{stitching policy}} that optimally selects the reward-maximizing horizon length and future sub-goal to achieve horizon-adaptive value stitching. This design enables direct estimation and compositional "stitching" of variable-length returns grounded in actionable sub-goal states, providing an accurate and greedily exploitable value-supervision signal for offline policy optimization. Across 50 tasks on OGBench, \textit{VAST} outperforms fixed-step, extended-horizon methods, and generative-value offline RL baselines, achieving strong performance particularly in high-complexity, long-horizon decision-making tasks. The official implementation is available at:~\url{https://github.com/Whiterrrrr/value-stitching}.
\end{abstract}
\section{Introduction}
\label{sec:intro}

The quintessence of reinforcement learning (RL) lies in value~\citep{sutton1988learning,watkins1992q,mnih2015human,silver2016mastering,silver2017mastering}: it evaluates how present choices shape future consequences, provides the key signal for value judgment~\citep{mdp,Pignatelli2023ASO,hung2019optimizing}, and ultimately steers both the direction of optimization and the quality of the learned policy~\citep{park2024valuelearningreallymain}. However, in long-horizon, complex task settings, conventional temporal-difference (TD) learning is highly susceptible to error compounding over the horizon due to myopic bootstrapping iterations on compressed scalar targets, which often yields biased value estimates and inefficient learning~\citep{mdp,park2024valuelearningreallymain,jaakkola1993convergence,tsitsiklis1996analysis,de2018multi,park2025horizonreductionmakesrl}. To mitigate this limitation, recent works have started pursuing horizon-extended backups~\citep{konidaris2011td_gamma,seo2024coarse,tian2025chunking,li2025reinforcement,li2025decoupledqchunking}, distributional value modeling~\citep{bellemare2017distributional,agrawalla2025floq,dongzheng2026value}, or structured value representations that encode temporal relationships~\citep{machado2023temporal}. While extending the effective horizon can improve value scalability by supplying richer and more structured supervision signals~\citep{park2025horizonreductionmakesrl}, existing approaches, which are typically formulated as fixed-step TD backups~\citep{konidaris2011td_gamma} or chunked action/reward sequences~\citep{konidaris2011td_gamma,seo2024coarse,li2025reinforcement,li2025decoupledqchunking}, still exhibit deficiencies due to their rigid, fixed-horizon parameterization.

\begin{figure}[t]
\centering
\includegraphics[width=\textwidth]{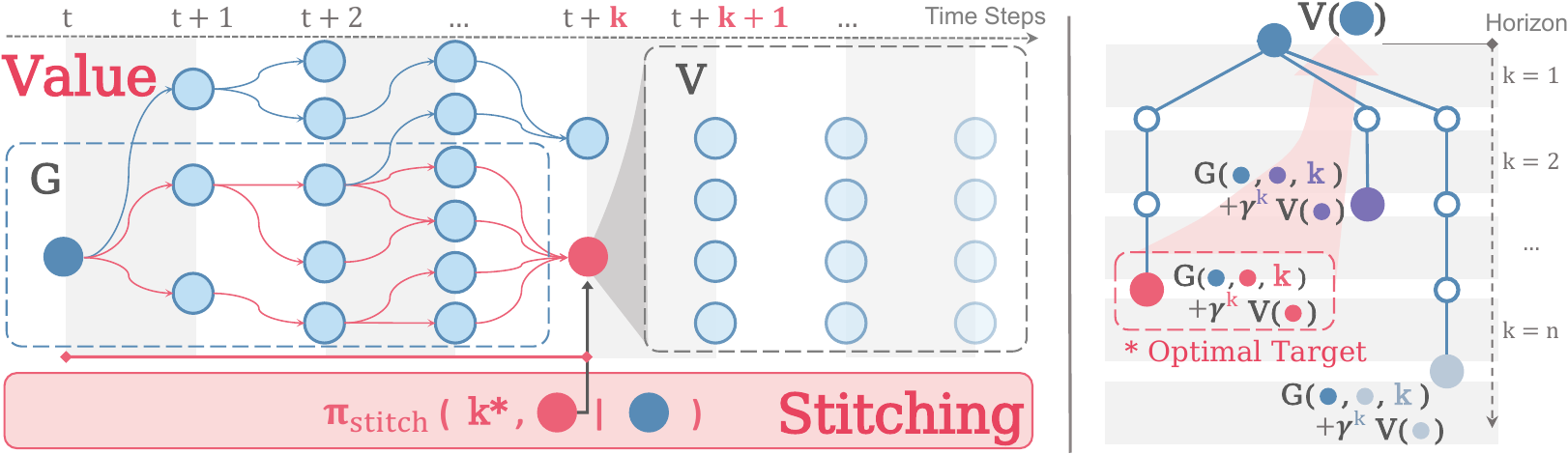}
\caption{\label{fig:illustration} \small \textbf{\textcolor{byebyebabeblue}{\textit{VAST}} compositionally stitches cumulative returns between adaptive-length, actionable sub-goals, thereby enhancing value learning optimality and stability.} (\textit{Left}): \textit{VAST} leverages the \textit{horizon-based cumulative return} ($G$-function) to measure the expected discounted return between two reachable states over a specific horizon, coupled with a stitching policy performing horizon-aware optimal sub-goal selection. (\textit{Right}): \textit{VAST} enables principled horizon adaptive value supervision guided by $G$-function with learnable backups.}
\vspace{-5pt}
\end{figure}

In stark contrast, human decision-making in real-world scenarios exhibits a high degree of dynamism and adaptability regarding credit assignment over task horizons. When planning for distant objectives, humans typically only conceptualize coarse-grained sub-goals with a rough estimate of horizon length (e.g., walk straight for about 100m and then turn right); while for immediate and critical decisions, they will focus exclusively on proximal, actionable consequences (e.g., decelerate and move left to avoid an obstacle). This mixed-horizon modeling requirement is akin to the idea of the "option" framework~\citep{SUTTON1999181,precup2000temporal}, which solves a semi-Markov decision process by optimizing temporally extended actions (i.e., macro-actions) with a terminal condition, thereby securing global optimality through localized decisions~\citep{stolle2002learning,bacon2016optioncriticarchitecture,riemer2019learningabstractoptions}. However, existing option-based methods remain structurally heavy and cumbersome to solve complex practical control problems. To simplify the modeling framework, many studies also resort to the hierarchical RL framework~\citep{park2025horizonreductionmakesrl,xu2022a,hiql,nachum2018data}, which jointly optimizes a high-level and a low-level policy at two distinct time scales. However, the hierarchical frameworks still can only capture two fixed modeling horizons, and the truly adaptive and practically deployable mixed-horizon offline RL paradigms remain largely underexplored.

To address this challenge, we propose  \textcolor{byebyebabeblue}{\textbf{\textit{VAST}}} (illustrated in Figure~\ref{fig:illustration}), a horizon-adaptive offline RL framework that equips value and policy learning with temporal adaptivity, while simultaneously offering improved value learning stability and policy optimality. 
At its core, \textit{VAST} integrates value optimization with a future-state- and horizon-conditioned \emph{\textbf{auxiliary value function}} (referred to as \textit{horizon-based cumulative return}, $G$-function) and a \emph{\textbf{stitching policy}} ($\pi_\text{stitch}$) that adaptively selects optimal sub-goal states and corresponding stitching horizons. Specifically, the $G$-function estimates the expected discounted return between arbitrary reachable state pairs under a specified horizon, enabling direct evaluation of actionable sub-goals. This naturally gives rise to a new \emph{stitching Bellman operator} that, together with the stitching policy, identifies the optimal target state for horizon-adaptive value backup.
For policy extraction, \textit{VAST} decomposes return maximization into the aforementioned \emph{\textbf{stitching policy}} $\pi_\text{stitch}$, which selects horizon-aware sub-goals that best steer behavior toward high-return futures, and an \emph{\textbf{execute policy}} $\pi_\text{exec}$, which is supervised trained directly on offline data and further optimized at test time via rejection sampling. The overall design enables both direct estimation and compositional stitching of returns anchored at actionable sub-goals, yielding supervision that is accurate, stable, and greedily exploitable for both value learning and low-level control. Across 50 OGBench~\citep{ogbench_park2025} tasks under the offline RL setting, \textit{VAST} performs the best or on par with $n$-step backup, Q-chunking methods, and generative-value-based baselines, with especially pronounced performance gains in complex long-horizon tasks.

\section{Preliminaries}
\label{sec:preliminary}

\paragraph{Problem formulation.} We consider the RL problem in a Markov Decision Process (MDP)~\citep{mdp}, defined by a tuple $\mathcal{M}:=\left(\mathcal{S},\mathcal{A},\mathcal{P},r, \gamma\right)$, where $\mathcal{S}$ and $\mathcal{A}$ denote the state and action space, $\mathcal{P}:\mathcal{S} \times \mathcal{A} \to \Delta(\mathcal{S})$ is the operator that represents the transition dynamics, $r:\mathcal{S} \times \mathcal{A} \to \mathbb{R}$ denotes the reward function, and $\gamma \in (0,1)$ is the discount factor. The goal of the RL problem is to find a policy $\pi:\mathcal{S} \to \Delta(\mathcal{A})$ that maximizes the expected return: $\mathbb{E}_{\pi}\left[\sum^{\infty}_{k=0}\gamma^{k}\cdot r(s_k,a_k) \right]$. In this work, we focus on the offline RL setting, the goal of which is to learn a policy under the fixed dataset \( \mathcal{D}=\{s_i, a_i, r_i, s_{i+1}\}_{i=1}^N\) collected by a behavioral policy $\mu$.

\paragraph{Single-step and multi-step value backups.} Approximate dynamic programming~\citep{powell2007approximate} is a central paradigm in RL, which maintains an action-value function $Q(s,a)$, or optionally a state-value function $V(s)$ for policy \(\pi\) under a given reward function \(r: \mathcal{S} \to \mathbb{R}\), defined as
\(Q(s, a):= \sum_{t \geq 0} \gamma^t \cdot \mathbb{E}_{a_t \sim \pi}[r(s_{t}, a_t) | s, a, \pi]\) and \(V(s):= \sum_{t \geq 0} \gamma^t \cdot \mathbb{E}_{a_t \sim \pi}[r(s_t, a_t) | s, \pi]\) respectively, with a further definition of the advantage function as $A(s, a) = Q(s, a)-V(s)$. In the offline RL setting, the value functions are typically learned with the following policy evaluation operator~\citep{nachum2017bridginggapvaluepolicy,kummar2020cql,Kostrikov2021OfflineRL,Xu2023OfflineRW}:
\begin{align}
(\mathcal{T}^\pi Q)(s, a) &:= r(s, a) + \gamma \cdot \mathbb{E}_{\mathcal{P}(s'|s,a)}\mathbb{E}_{a'\sim \pi}\left[Q(s',a')\right], \label{equ:1step_q}\\
(\mathcal{T}^\pi V)(s) &:= \mathbb{E}_{a\sim \pi} \left[ r(s, a) + \gamma \cdot \mathbb{E}_{\mathcal{P}(s'|s,a)} \left[ V(s') \right] \right].  \label{equ:1step_v}
\end{align}
The \(Q\) and \(V\)-function are typically learned by the conventional single-step objectives $\min_{Q} J(Q) =\frac{1}{2} \mathbb{E}_{(s, a) \sim \mathcal{D}}\left[(\mathcal{T}^\pi Q - Q)(s, a)^2 \right]$ and $\min_{V} J(V)=\frac{1}{2} \mathbb{E}_{s \sim \mathcal{D}}\left[(\mathcal{T}^\pi V - V)(s)^2\right]$, respectively. To mitigate error accumulation and improve learning stability, some early and recent works enforce stronger supervision by adopting the $n$-step backup~\citep{mdp,park2025horizonreductionmakesrl,konidaris2011td_gamma,watkins1989learning,peng1994incremental,thomas2015policy}:
\begin{align} 
(\mathcal{T}^\pi_\text{n-step}) Q(s_t, a_t) &:= \sum_{t'=t}^{t+n-1} \left[\gamma^{t'-t} r(s_{t'}, a_{t'})\right] + \gamma^n \cdot \mathbb{E}_{s_{t+n}\sim \mathcal{P}}\mathbb{E}_{a_{t+n}\sim \pi}\left[Q(s_{t+n},a_{t+n})\right], \label{equ:nstep_q} \\
(\mathcal{T}^\pi_\text{n-step} V)(s_t) &:= \mathbb{E}_{a_t\sim \pi} \left[\sum_{t'=t}^{t+n-1} \left[\gamma^{t'-t} r(s_{t'}, a_{t'})\right] + \gamma^n \cdot \mathbb{E}_{s_{t+n}\sim\mathcal{P}} \left[ V(s_{t+n}) \right] \right]. \label{equ:nstep_v}
\end{align}

\paragraph{Chunked value function.} Another series of methods tackles the value bootstrapping scheme with action chunking~\citep{seo2024coarse,tian2025chunking,li2025reinforcement,li2025decoupledqchunking}, which extends the value backup from the action space to the action sequence level:
\begin{equation} \label{equ:chunking_q}
(\mathcal{T}^\pi_\text{chunking} Q)(s_t, a_{t:t+n}) := \sum_{t'=t}^{t+n-1} \left[\gamma^{t'-t} r(s_{t'}, a_{t'})\right] + \gamma^n \cdot \mathbb{E}_{s_{t+n}\sim\mathcal{P}}\mathbb{E}_{a_{t+n:t+2n}\sim \pi}\left[Q(s_{t+n},a_{t+n:t+2n})\right].
\end{equation}
Compared with the $n$-step backup, the chunked value naturally suits the popular action-chunking design in imitation learning~\citep{zhao2023learning} and mitigates off-policy bias in previous methods~\citep{li2025reinforcement}.
However, both paradigms still suffer from some deficiencies due to the fixed horizon length modeling. $n$-step backups are inherently conservative in value estimation (the $n$-step return is estimated from the data collection policy rather than the optimal policy), whereas action chunking greatly expands the action space, creating potential optimization challenges.

\paragraph{Option-value function.} The option framework~\citep{SUTTON1999181,precup2000temporal} models a Markovian option $\omega \in \Omega$ as a triple $\left(\mathcal{I}_\omega,\pi_\omega,\beta_\omega\right)$, where $\mathcal{I}_\omega \subseteq \mathcal{S}$ is the initiation set, $\pi_\omega:\mathcal{S}\times\mathcal{A}\rightarrow[0,1]$ is the intra-option policy over primitive actions, and $\beta_\omega:\mathcal{S}\rightarrow[0,1]$ is the termination function. Augmenting an MDP with a fixed set of options $\Omega$ (by a policy over options \(\pi_\Omega(\omega \;| \;s)\)) induces a semi-Markov decision process (SMDP)~\citep{SUTTON1999181,puterman2014markov}, on which the option-value function is defined as $
Q_\Omega(s_t,\omega):=\mathbb{E}_{\pi_{\omega}}\left[\sum_{t'=t}^{t+k-1}\gamma^{t'-t}\,r(s_{t'},a_{t'})\;+\;\gamma^{k}\,V_\Omega(s_{t+k})\; | \;s_t,\omega\right]$,
where we use $k$ to denote the random duration at which $\omega$ terminates according to $\beta_\omega$, and $V_\Omega$ is the value of options at the call-and-return state $s_{t+k}$; the expectation is taken jointly over the action sequence and the termination event. While this formulation yields a principled treatment of temporally extended behavior, existing option-based approaches remain architecturally heavy, which poses considerable challenges for practical control tasks~\citep{stolle2002learning,bacon2016optioncriticarchitecture,riemer2019learningabstractoptions,vezhnevets2017feudal}. In this work, we borrow the insight from the option framework, but establish a more effective and simpler value learning framework with flexible horizon length modeling capability.
More discussions of various value learning schemes and hierarchical RL frameworks are provided in Section~\ref{sec:related}.
\section{Modeling Cumulative Return for Adaptive Horizon Length}
\label{sec:grounded_return}
The Bellman updates can accumulate substantial error through recursive bootstrapping on intermediate predicted targets (Eq.~(\ref{equ:1step_q})), especially in long-horizon tasks~\citep{mdp,jaakkola1993convergence,de2018multi,park2025horizonreductionmakesrl,li2025reinforcement,li2025decoupledqchunking}. As discussed in the previous section, multi-step value learning schemes such as $n$-step backups (Eq.~(\ref{equ:nstep_q})) and chunked value function (Eq.~(\ref{equ:chunking_q})) can strengthen supervision by extending the horizon of grounded reward signals, but are limited by fixed-step backups. In the following example, we illustrate the pitfalls of fixed-step value updates and then present a neat solution for horizon-adaptive value stitching.

\begin{wrapfigure}{r}{0.35\textwidth}
    \centering
    \includegraphics[width=0.33\textwidth]{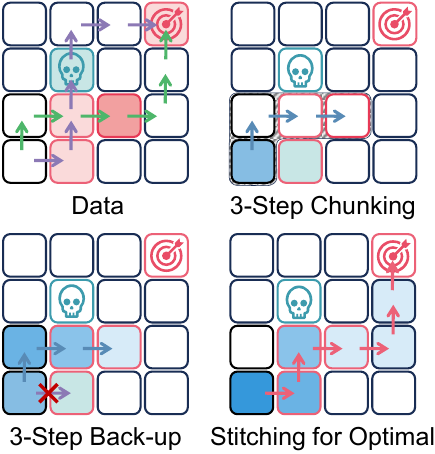} 
    \caption{ \small \textbf{4x4 grid navigation.} Toy example requiring navigation from the bottom-left to the top-right. The data include sub-optimal paths only ({\color[HTML]{4EB66A} \textbf{green path}} and {\color[HTML]{7030A0} \textbf{purple path}}); the environment is specified by \textcolor{lovelyred}{\textbf{dark red}}, \textcolor{tinyred}{\textbf{light red}}, and {\color[HTML]{3D9BAD} \textbf{skull}} blocks, indicating high-, low-, and punitive-reward. The \textcolor{byebyebabeblue}{\textbf{depth of blue in blocks}} indicates the value estimates from fixed-step/chunked value backups, and the optimal stitching value.}
    \label{fig:toycase}
\end{wrapfigure}

\paragraph{The limitation of fixed-horizon modeling.}
Central to the identified problem of fixed-horizon value-learning is that, while it may yield more robust value estimates, it is poorly suited to extracting policies for real-time, stepwise execution. When the backup sequence length is mismatched, crucial single-step signals can be washed out by sequence-level averaging, leading to the loss of precisely the information needed for accurate immediate decisions. 
This intuitive failure mode is readily illustrated by a simple toy example: the $4\times4$ grid navigation task in Figure~\ref{fig:toycase}. The agent starts at the bottom-left corner and aims to reach the goal at the top-right corner by moving only upward or rightward. Rewards for each reach are encoded by cell color shown in the data subgraph: dark red denotes high reward (e.g., 2), light red denotes lower reward (e.g., 1), green skull denotes penalty (e.g., -2), and blank cells have zero reward. The dataset contains only suboptimal trajectories (the green and purple paths).

Under an improper fixed horizon (e.g., the 3-step horizon in this example), the agent fails to recover optimal behavior. The intensity of the blue shading indicates the learned value under each method, with darker regions corresponding to higher values. Chunked value updates favor the first three transitions along the green path, but fail to guide the agent to the destination.
In contrast, $n$-step backups suppress the transition from $(0,0)$ to $(1,0)$ (the first purple transition in data) because it is contaminated by the later negative reward at $(1,2)$; and notably, both fixed-step methods lead to a negative value on $(1,0)$ with a discount factor $\gamma \approx 1$, shown with a skull shading. However, the optimal solution requires stitching together favorable short segments from different trajectories (red path). Refer to Appendix~\ref{appendix:toy_details} for further details. Although deliberately simple, this toy example can be viewed as a local subproblem embedded within a much larger maze. It shows that while coupling value updates to a fixed horizon can better stabilize long-horizon learning scenarios, it simultaneously obscures critical short-range decisions, thereby trapping learning in poor local optima. Stability should not be bought at the cost of erasing decisive short-term signals, motivating adaptive horizon modeling that preserves both robustness and flexibility.

\paragraph{Towards adaptive horizon composition.}
The failure of fixed-horizon methods raises a fundamental question: \textit{how should the backup horizon be selected optimally?} Addressing this issue requires moving beyond value estimates rigidly tied to transition sequences in the data buffer. Instead, we need a more general operator that can evaluate the cumulative return between any pair of states under a \emph{dynamically chosen horizon}. This insight motivates the introduction of the following definition:

\begin{definition}[Horizon-Conditioned Cumulative Return]
Let $\tau = (s_0, a_0, r_0, s_1, a_1, r_1, \dots)$ be the trajectory induced by a policy $\pi$ starting at state $s_0 = s$.
For a target state $s^+$ reached exactly at step $k$ (i.e., $s_k = s^+$), we define the horizon-conditioned cumulative return as:
\begin{equation}
G(s, s^+, k) := \mathbb{E}_{\tau \sim \pi}\left[\sum_{t=0}^{k-1} \gamma^t r(s_t, a_t) \;\middle|\; s_0 = s,\; s_k = s^+ \right].
\end{equation}
Given $|r(s,a)| \le R_{\max}$ and $\gamma \in (0,1)$, then $G$ is strictly bounded as $|G(s, s^+, k)| < \frac{R_{\max}}{1-\gamma}$.
\end{definition}

We refer to $G(s, s^+, k)$ as the \textit{Horizon-Conditioned Cumulative Return Function}. By conditioning on both the terminal state $s^+$ and the exact horizon $k$, $G$ defines a value mapping relationship between arbitrary $k$-step reachable state pairs and the policy-induced distribution.
This formulation breaks the dependence of value targets on fixed local backup segments and further supports flexible temporal composition across extended sub-trajectory segments, as shown in Theorem~\ref{thm:compositional}, thereby forming the foundation for adaptive horizon selection in value learning.

\begin{theorem}[Temporal Compositionality]
\label{thm:compositional}
For any target state $s_k$ reachable from $s$ at step $k$, and any intermediate state $s_i$ visited at step $i$ with $0<i<k$ along a trajectory induced by policy $\pi$, the horizon-conditioned cumulative return admits the decomposition:
\begin{equation}
G(s, s_k, k) = \mathbb{E}_{i \sim \mathrm{Unif}\{1,\dots,k-1\},\, s_i \sim \pi}\left[G(s, s_i, i) + \gamma^i G(s_i, s_k, k-i)\right].
\end{equation}
\end{theorem}
The proof is provided in Appendix~\ref{appendix:proof_compostional}.

Theorem~\ref{thm:compositional} formalizes the key principle underlying temporal value composition. It shows that a long-horizon return can be decomposed into a sequence of shorter and composable return segments. This property is particularly important for a constructive function that learns value relationships across multiple temporal scales among reachable states.

\textbf{Learning objective of \(G\)-function.}
To learn a parameterized \(G\)-function \(G_\psi(s, s^+, k)\) from an offline dataset \(\mathcal{D}\), we optimize two complementary objectives: (i) a Monte-Carlo regression objective for direct supervision, and (ii) a temporal compositional objective induced by Theorem~\ref{thm:compositional} to enforce multi-scale consistency:
\begin{align}
\mathcal{L}_{\mathrm{MC}}(\psi) =& \mathbb{E}_{\tau \sim \mathcal{D},\, k \sim \mathrm{Unif}\{1,\dots,K\}}
\bigl[ \bigl( G_\psi(s, s_k, k) - \sum_{t=0}^{k-1}\gamma^{t} r(s_t, a_t) \bigr)^2 \bigr], \label{eql:g_mc} \\
\mathcal{L}_{\mathrm{comp}}(\psi)
=& \mathbb{E}_{\tau \sim \mathcal{D},\, k \sim \mathrm{Unif}\{2,\dots,K\},\, i \sim \mathrm{Unif}\{1,\dots,k-1\}}
\Bigl[ \Bigl( G_\psi(s, s_k, k) \notag\\
& - \bigl( G_\psi(s, s_i, i) + \gamma^{i} G_\psi(s_i, s_k, k-i) \bigr) \Bigr)^2 \Bigr], \label{eql:g_comp}
\end{align}
where trajectory \(\tau\) is sampled uniformly from offline dataset \(\mathcal{D}\), \(k\) is sampled uniformly from \(\{1,\dots,K\}\) with \(K\) denoting the maximum horizon, starting state $s_0=s$ is randomly sampled from the \(\tau\), and \(i\) is sampled uniformly from \(\{1,\dots,k-1\}\). The overall objective is
\begin{equation} \label{eql:g_loss}
\mathcal{L}_{G}(\psi) = \mathcal{L}_{\mathrm{MC}}(\psi) + \lambda \mathcal{L}_{\mathrm{comp}}(\psi).
\end{equation}
The MC objective provides direct return supervision, while the compositional objective regularizes \(G_\psi\) to satisfy temporally consistent value composition, improving generalization across horizons. Once learned, \(G_\psi\) is no longer restricted to fixed backup segments from the offline data buffer, thereby enabling adaptive, target-state-anchored value maximization..
\section{Horizon Adaptive Offline Policy Learning}
\label{sec:rl_vs}
With the \(G\)-function established in Section~\ref{sec:grounded_return}, we now turn to its use in offline policy optimization. To this end, we propose \textcolor{byebyebabeblue}{\textbf{\textit{VAST}}}, which is built upon a novel horizon-adaptive stitching Bellman operator, as well as a hierarchical policy optimization scheme involving a \emph{stitching policy} $\pi_\text{stitch}$ and an \emph{execution policy} $\pi_\text{exec}$. In the following, we describe the details of our algorithm design.

\paragraph{Stitching Bellman operator.}
Note that with the previously defined $G$-function, for a given policy $\pi$, we can equivalently express the state-value function $V^\pi$ as (assuming sufficiently large $K$):
\begin{equation} \label{eql:v_def}
V^\pi(s) = \mathbb{E}_{k \sim \mathrm{Unif}\{1,...,K\}, s_k \sim \Gamma^{\pi}(s, k)} \left[ G(s,s_k,k) + \gamma^k \cdot V^\pi(s_k) \right].
\end{equation}
where $\Gamma^{\pi}(s, k)$ denotes the state distribution induced by rolling out with $\pi$ starting from $s$ for $k$ steps.
And in particular, when $k=1$, the above reduces to the special one-step horizon case,
\begin{equation}
V^\pi(s_t) = \mathbb{E}_{s_{t+1}\sim \mathcal{P}(\cdot|s_t, a_t), a_t\sim\pi} \left[ G(s_t,s_{t+1},1) + \gamma \cdot V^\pi(s_{t+1}) \right],
\end{equation}
which recovers the standard Bellman equation with one-step anchoring (as $G^\pi(s_t,s_{t+1},1)=r(s_t,a_t)$). This observation naturally motivates us to introduce a new \emph{stitching} Bellman operator:
\begin{equation}
\label{eq:stitch_op}
(\mathcal{T}^\pi_{\text{stitching}} V)(s) = \max_{k \in \{1,...,K\}, s_k \sim \Gamma^{\pi}(s, k)}  G(s, s_k, k) + \gamma^k \cdot V(s_k).
\end{equation} 
Specifically, the $\max$ operator is a \emph{joint} maximization over the horizon $k\in\{1,\dots,K\}$ and the target states $s_k$ in the support of the $k$-step reachable distribution $\Gamma^{\pi}(s,k)$. Therefore, only the feasible set contains state-horizon pairs $(k,s_k)$ that are actually reachable from $s$ in exactly $k$ steps under $\pi$ will be evaluated.
Repeatedly applying this operator on $V$ will converge to a unique fixed point $V^*$, as it can be proved to be a $\gamma$-contraction mapping:
\begin{theorem}[Contraction Mapping of the Stitching Bellman Operator]
\label{thm:contraction}
For any two bounded value functions $V_1, V_2$, there exists a horizon length $\tilde{k}\sim \{1,\cdots, K\}$ such that the operator $\mathcal{T}_{\text{stitch}}$ is a $\gamma$-contraction in the infinity norm:
\begin{equation}
\|\mathcal{T}_{\text{stitch}} V_1 - \mathcal{T}_{\text{stitch}} V_2\|_\infty \le \gamma^{\tilde{k}} \| V_1 - V_2 \|_\infty \leq \gamma \| V_1 - V_2 \|_\infty.
\end{equation}
\end{theorem}
See Appendix~\ref{appendix:proof_contraction} for the proof. With the state-value function $V$ estimated using the stitching Bellman operator, the action-value function $Q$ can be learned via the standard Bellman backup:
\begin{equation} \label{eql:q_loss}
(\mathcal{T}^\pi Q)(s,a) := r(s, a) + \gamma \cdot \mathbb{E}_{s' \sim P(\cdot \mid s)} \left[V(s') \right],
\end{equation}

\paragraph{Value stitching via $\pi_\text{stitch}$.} The stitching Bellman operator offers a new way to bootstrap the value estimates based on the optimal horizon lengths and target states. However, practically achieving such horizon-adaptive value learning in the offline setting requires some additional designs. In this study, we introduce a stitching policy $\pi_\text{stitch}(\cdot|s): \mathcal{S} \rightarrow (\mathbb{Z}^+_{\leq K}, \mathcal{S})$  that simultaneously selects 1) the best horizon $k$ for value stitching and 2) the reward-maximizing sub-goal state $s_k$ based on current state $s$. Based on Eq.(\ref{eq:stitch_op}), the offline learning objective for the $V$-function can be written as:
\begin{equation}
\label{eq:stitch_v_obj}
V = \underset{V}{\arg\min} \mathbb{E}_{s\sim \mathcal{D}} \Big[\Big[\max_{(k,s_k) \sim \pi_\text{stitch}(\cdot|s)}  \underbrace{G(s, s_k, k) + \gamma^k V(s_k)}_{V_{\text{target}}^{\pi}}\Big] -V(s)\Big]^2
\end{equation}
The above value objective adaptively regresses the value target based on the optimal horizon length specified by the $G$-function, thereby achieving horizon-adaptive \textit{value stitching}. However, evaluating the $\max$ operator can be difficult for the continuous state space. Therefore, we leverage the expectile regression objective as introduced in IQL~\citep{Kostrikov2021OfflineRL} to mitigate this difficulty:
\begin{equation} \label{eql:v_loss}
\mathcal{L}_V = \mathbb{E}_{s \sim \mathcal{D}, (k, s_k) \sim \pi_\text{stitch}(\cdot | s)} \left[ L_{2}^{\tau} \left(V_{\text{target}}^{\pi} - V(s)\right) \right],
\end{equation}
where $ L_{2}^{\tau}(u) = |\tau - \mathbb{I}(u<0)|u^2$, with $\tau > 0.5$ serving as the expectile parameter.
$V_{\text{target}}^{\pi}$ is evaluated at $(k,s_k)\sim\pi_\text{stitch}(\cdot\mid s)$. The policy weights below use $V_{\text{target}}^{\mathcal{D}}$, the same expression $G(s,s_k,k)+\gamma^{k}V(s_k)$ on a same-trajectory pair $(k,s_k)\sim\mathcal{D}$.

With the value learning objective established, we also need a simple approach to optimize the stitching policy $\pi_\text{stitch}$ while adhering to the data distribution. Specifically, we choose the advantage weighted regression (AWR)~\citep{peng2019advantage} to extract $\pi_\text{stitch}$, which is widely adopted in many offline RL methods~\citep{Kostrikov2021OfflineRL,Xu2023OfflineRW,garg2023extreme,zheng2024safe,liang2026dipole}. It can be perceived as regressing the optimal solution of an RL objective that maximizes the advantage value with a reverse Kullback-Leibler (KL) divergence constraint on the data distribution~\citep{Xu2023OfflineRW,peng2019advantage}. But in our case, we evaluate and maximize over a special \textit{state-advantage} conditioned on both horizon length and sub-goal state, i.e.,
\begin{equation} \label{eql:awr_pi_stitch}
\mathcal{L}_{\pi_\text{stitch}} = \mathbb{E}_{(s,k,s_k)\sim\mathcal{D} }\left[ \exp\left(\beta \cdot \left(V_{\text{target}}^{\mathcal{D}} - V(s)\right)\right) \cdot \log \pi_\text{stitch}(k,s_k|s) \right],
\end{equation}
Moreover, to further enhance the expressiveness and generalization capability of the stitching policy, as well as to improve the robustness of value backups, we parameterize $\pi_\text{stitch}$ as a diffusion/flow model~\citep{sohl2015deep,ho2020denoising,lipman2023flow}. Concretely, $\pi_\text{stitch}$ jointly predicts the optimal horizon $k$ and sub-goal $s_k$ in the continuous space $(\mathbb{R}_{+}, \mathcal{S})$. Let $\epsilon_\text{stitch}$ denote the neural network parameterizing the stitching policy, which is trained to predict the injected noise over the dataset samples. It has been shown and adopted in many diffusion-based RL studies~\citep{zheng2024safe,liang2026dipole,zheng2025towards,ma2025efficient} that the following weighted diffusion loss can be used to optimize the above AWR-style objective:
\begin{equation} \label{eql:weighted_high}
\mathcal{L}_{\epsilon_\text{stitch}}=\mathbb{E}_{t\sim U[0,1],\epsilon\sim\mathcal{N}(\mathbf{0},\mathbf{I}),(s,k,s_k)\sim\mathcal{D}}\left[ \exp\left(\beta \cdot \left(V_{\text{target}}^{\mathcal{D}} - V(s)\right)\right) \cdot \left\Arrowvert \epsilon-\epsilon_\text{stitch}\left(k^t,s^t_k,s,t\right) \right \Arrowvert^2 \right]
\end{equation}
where $(k^t, s^t_k) = \alpha^t (k, s_k) + \sigma^t \epsilon$ with $\alpha^t$ and $\sigma^t$ being predefined noise schedules. Empowered by the stitching policy, \textit{VAST} not only enables effective state-value learning directly from offline data with only state-dependent transitions, but also flexibly extracts globally relevant goal information for better potential downstream utilization.

\paragraph{Optimized action extraction via $\pi_\text{exec}$.} With the optimal sub-goal state predicted from the learned stitching policy $\pi_\text{stitch}$, we can learn a lower-level \emph{execution policy} $\pi_\text{exec}: \mathcal{S} \times \mathcal{S} \rightarrow \mathcal{A}$ for control under its guidance. Due to the unique learning procedure of $\pi_\text{stitch}$, it can provide both reachable and reward-maximizing sub-goals. This is akin to the human decision-making process that performs both coarse global-level planning as well as adaptive short-horizon adjustment. We learn $\pi_\text{exec}(a|s,s^+)$ as a flow policy in a simple goal-conditioned imitation learning manner that directly supervises it with offline data samples, conditioned on both current state $s$ and randomly sampled future goal state $s^+$ in the same data trajectory. This bypasses the out-of-distribution (OOD) issue during training while enjoying a stable learning process.
Specifically, we omit $k$ to decouple $\pi_\text{exec}$ from the accuracy of $\pi_\text{stitch}$, thereby avoiding a train-inference mismatch, where at training time $k$ comes from sub-optimal data, while at inference time it comes from the optimized $\pi_\text{stitch}$.

To further improve the optimality of the output action, similar to previous methods~\citep{fql_park2025, Mao2024DiffusionDICEID}, we filter the sampled actions with test-time rejection sampling conditioned on the sub-goal $s_k\sim \pi_\text{stitch}(\cdot | s)$:
\begin{equation} \label{eq:topk}
a^\star \triangleq {\arg\max}_{a \in \{a^{(1)}, \ldots, a^{(N_\text{rej})} \sim \pi_\text{exec}(\cdot | s, s_k)\}} Q(s,a),
\end{equation}
where $N_\text{rej}$ is the action number to be sampled. $Q$-function is learned by minimizing the standard Bellman error with Eq.(\ref{eql:q_loss}). The above formulation performs implicit $Q$-value maximization by resampling the distribution learned by the expressive goal-conditioned BC flow policy, with instructive and advantageous sub-goal information additionally injected.
See Appendix~\ref{appendix:discussion} for more discussions, and Appendix~\ref{appendix:implementation} for additional implementation details.
\section{Experiments}
\label{sec:experiments}

Briefly, \textit{VAST} employs \emph{an auxiliary value function} to improve \emph{(1) value-learning} and further the \emph{(2) policy-execution optimality and robustness}. Accordingly, we expect \textit{VAST} to achieve strong performance on long-horizon complex planning tasks, outperforming strong value-scaling approaches. 

\paragraph{Environmental setup.} We evaluate \textit{VAST} on the OGBench task suite~\citep{ogbench_park2025}. OGBench is designed for goal-conditioned offline RL, with challenging robotic locomotion and manipulation tasks, including whole-body humanoid control, maze navigation, and object manipulation. Our experiments are conducted in the \texttt{singletask} variant for offline RL, covering 30 state-based tasks across 6 widely used domains, denoted as \textit{Basic}, and 20 state-based tasks across 4 highly challenging domains that increase the planning complexity, denoted as \textit{Hard}, with \textbf{a total of 50 tasks on OGBench}. See Appendix~\ref{appendix:experimental_env} for environmental details.

\paragraph{Baselines.} We select 8 baselines for comprehensive evaluation, which could be categorized into value learning methods with 1-step/multi-step TD, chunking, and generative value:

\begin{itemize}[leftmargin=*]
\item \emph{1-step TD.}\quad We select 1) \textit{IQL}~\citep{Kostrikov2021OfflineRL}: a typical weighted regression offline RL method with Gaussian policy and expectile regression for value learning; 2) \textit{FQL}~\citep{fql_park2025}: a behavior-regularized actor-critic (AC) variant that uses flow policy distillation and shows strong performance on OGBench; 3) \textit{IFQL}~\citep{fql_park2025}: a variant of \textit{IQL} that uses flow policy with rejection sampling during inference.

\item \emph{Multi-step TD.}\quad We select 1) \textit{FQL-n} and 2) \textit{IFQL-n}, the variants of \textit{FQL} and \textit{IFQL} with fixed-n-step value backups. Following~\citet{li2025reinforcement}, we fix the horizon to 5.

\item \emph{Chunking.}\quad We select the most typical chunked value baseline \textit{QC}~\citep{li2025reinforcement}: Q-learning on a fixed-size action sequence space.

\item \emph{Generative value.} To illustrate \textit{VAST}'s effectiveness on value learning and guidance, we further select the baselines that specifically focus on value scaling. This includes 1) \textit{VALUE-FLOWS}~\citep{dongzheng2026value}: a novel distributional RL method that estimates return variance using a flow-matching objective and learns a flow policy; 2) \textit{FLOQ}~\citep{agrawalla2025floq}: it parameterizes the critic as a learned velocity field, estimates the Q-value through numerical integration, and learns a flow policy.
\end{itemize}

See Appendix~\ref{appendix:experimental_baselines} for additional baseline information.

\begin{table}[t]
\centering
\caption{\label{table:ogbench_common} \small \textbf{OGBench Results.} We report the aggregate score on all single tasks for each category, averaging over 8 random seeds. \textit{VAST} achieves the best or near-best performance against other baselines across 10 domains, 50 tasks in total. The IQL results are reported from~\citet{fql_park2025}, thereby eliminating the $\pm$ sign.}
\resizebox{1\linewidth}{!}{\scriptsize
\begin{tabular}{lccccccccc}
\toprule
& \multicolumn{3}{>{\columncolor[HTML]{EEEEEE}}c}{\textbf{1-step TD}}
& \multicolumn{2}{>{\columncolor[HTML]{EEEEEE}}c}{\textbf{n-step TD}}
& \multicolumn{1}{>{\columncolor[HTML]{EEEEEE}}c}{\textbf{Chunking}}
& \multicolumn{2}{>{\columncolor[HTML]{EEEEEE}}c}{\textbf{Generative}}
& \multicolumn{1}{>{\columncolor[HTML]{EEEEEE}}c}{\textbf{Stitching}} \\
\cmidrule(lr){2-4} \cmidrule(lr){5-6} \cmidrule(lr){7-7} \cmidrule(lr){8-9} \cmidrule(lr){10-10}
(5 tasks each) & IQL & FQL & IFQL & FQL-n & IFQL-n & QC & VALUE-FLOWS & FLOQ & VAST \\
\midrule
\texttt{scene}              & $\textrm{28}$ 
                            & $\textrm{57} \textcolor{lightgray}{\pm \textrm{7}}$ 
                            & $\textrm{46} \textcolor{lightgray}{\pm \textrm{12}}$ 
                            & $\textrm{42} \textcolor{lightgray}{\pm \textrm{8}}$ 
                            & $\textrm{32} \textcolor{lightgray}{\pm \textrm{2}}$ 
                            & $\textrm{60} \textcolor{lightgray}{\pm \textrm{1}}$
                            & $\textrm{60} \textcolor{lightgray}{\pm \textrm{5}}$
                            & $\textrm{57} \textcolor{lightgray}{\pm \textrm{5}}$ 
                            & \colorbox{tableblue}{$\textrm{61} \textcolor{lightgray}{\pm \textrm{3}}$}
\\
\texttt{cube-single}        & $\textrm{83}$ 
                            & $\textrm{97} \textcolor{lightgray}{\pm \textrm{4}}$
                            & $\textrm{87} \textcolor{lightgray}{\pm \textrm{5}}$ 
                            & $\textrm{97} \textcolor{lightgray}{\pm \textrm{3}}$
                            & $\textrm{82} \textcolor{lightgray}{\pm \textrm{6}}$ 
                            & $\textrm{93} \textcolor{lightgray}{\pm \textrm{5}}$ 
                            & $\textrm{95} \textcolor{lightgray}{\pm \textrm{4}}$ 
                            & $\textrm{96} \textcolor{lightgray}{\pm \textrm{5}}$ 
                            & \colorbox{tableblue}{$\textrm{98} \textcolor{lightgray}{\pm \textrm{2}}$}
\\
\texttt{cube-double}        & $\textrm{7}$ 
                            & $\textrm{30} \textcolor{lightgray}{\pm \textrm{8}}$ 
                            & $\textrm{12} \textcolor{lightgray}{\pm \textrm{5}}$ 
                            & $\textrm{4} \textcolor{lightgray}{\pm \textrm{3}}$ 
                            & $\textrm{4} \textcolor{lightgray}{\pm \textrm{3}}$ 
                            & $\textrm{67} \textcolor{lightgray}{\pm \textrm{7}}$
                            & $\textrm{61} \textcolor{lightgray}{\pm \textrm{11}}$ 
                            & $\textrm{49} \textcolor{lightgray}{\pm \textrm{16}}$ 
                            & \colorbox{tableblue}{$\textrm{75} \textcolor{lightgray}{\pm \textrm{6}}$}
\\
\texttt{puzzle-4x4}         & $\textrm{7}$ 
                            & $\textrm{15} \textcolor{lightgray}{\pm \textrm{5}}$ 
                            & $\textrm{27} \textcolor{lightgray}{\pm \textrm{8}}$ 
                            & $\textrm{27} \textcolor{lightgray}{\pm \textrm{6}}$ 
                            & $\textrm{12} \textcolor{lightgray}{\pm \textrm{4}}$ 
                            & $\textrm{29} \textcolor{lightgray}{\pm \textrm{10}}$ 
                            & $\textrm{21} \textcolor{lightgray}{\pm \textrm{8}}$ 
                            & $\textrm{31} \textcolor{lightgray}{\pm \textrm{9}}$ 
                            & \colorbox{tableblue}{$\textrm{46} \textcolor{lightgray}{\pm \textrm{10}}$}
\\
\texttt{antmaze-large}      & $\textrm{53}$ 
                            & $\textrm{82} \textcolor{lightgray}{\pm \textrm{7}}$ 
                            & $\textrm{32} \textcolor{lightgray}{\pm \textrm{17}}$ 
                            & $\textrm{80} \textcolor{lightgray}{\pm \textrm{9}}$ 
                            & $\textrm{27} \textcolor{lightgray}{\pm \textrm{7}}$ 
                            & $\textrm{6} \textcolor{lightgray}{\pm \textrm{5}}$ 
                            & $\textrm{48} \textcolor{lightgray}{\pm \textrm{15}}$ 
                            & \colorbox{tableblue}{$\textrm{93} \textcolor{lightgray}{\pm \textrm{6}}$} 
                            & $\textrm{81} \textcolor{lightgray}{\pm \textrm{7}}$
\\
\texttt{antmaze-giant}      & $\textrm{4}$ 
                            & $\textrm{11} \textcolor{lightgray}{\pm \textrm{16}}$ 
                            & $\textrm{0} \textcolor{lightgray}{\pm \textrm{2}}$ 
                            & \colorbox{tableblue}{$\textrm{37} \textcolor{lightgray}{\pm \textrm{16}}$} 
                            & $\textrm{0} \textcolor{lightgray}{\pm \textrm{0}}$ 
                            & $\textrm{0} \textcolor{lightgray}{\pm \textrm{0}}$ 
                            & $\textrm{3} \textcolor{lightgray}{\pm \textrm{3}}$ 
                            & $\textrm{29} \textcolor{lightgray}{\pm \textrm{24}}$ 
                            & $\textrm{20} \textcolor{lightgray}{\pm \textrm{6}}$
\\
\cmidrule(lr){2-10}
Overall (\textit{Basic})    & $\textrm{30}$ 
                            & $\textrm{49} \textcolor{lightgray}{\pm \textrm{9}}$ 
                            & $\textrm{34} \textcolor{lightgray}{\pm \textrm{10}}$ 
                            & $\textrm{48} \textcolor{lightgray}{\pm \textrm{9}}$ 
                            & $\textrm{26} \textcolor{lightgray}{\pm \textrm{4}}$ 
                            & $\textrm{43} \textcolor{lightgray}{\pm \textrm{6}}$ 
                            & $\textrm{48} \textcolor{lightgray}{\pm \textrm{9}}$ 
                            & $\textrm{59} \textcolor{lightgray}{\pm \textrm{13}}$ 
                            & \colorbox{tableblue}{$\textrm{64} \textcolor{lightgray}{\pm \textrm{6}}$}
\\
\midrule
\texttt{cube-triple}        & $\textrm{-}$ 
                            & $\textrm{1} \textcolor{lightgray}{\pm \textrm{4}}$ 
                            & $\textrm{1} \textcolor{lightgray}{\pm \textrm{1}}$ 
                            & $\textrm{0} \textcolor{lightgray}{\pm \textrm{0}}$ 
                            & $\textrm{0} \textcolor{lightgray}{\pm \textrm{0}}$ 
                            & $\textrm{4} \textcolor{lightgray}{\pm \textrm{5}}$ 
                            & $\textrm{13} \textcolor{lightgray}{\pm \textrm{10}}$ 
                            & $\textrm{4} \textcolor{lightgray}{\pm \textrm{5}}$ 
                            & \colorbox{tableblue}{$\textrm{17} \textcolor{lightgray}{\pm \textrm{8}}$}
                                        \\
\texttt{cube-quadruple}     & $\textrm{-}$ 
                            & $\textrm{0} \textcolor{lightgray}{\pm \textrm{0}}$ 
                            & $\textrm{0} \textcolor{lightgray}{\pm \textrm{0}}$ 
                            & $\textrm{0} \textcolor{lightgray}{\pm \textrm{0}}$ 
                            & $\textrm{0} \textcolor{lightgray}{\pm \textrm{0}}$ 
                            & $\textrm{0} \textcolor{lightgray}{\pm \textrm{0}}$ 
                            & $\textrm{0} \textcolor{lightgray}{\pm \textrm{0}}$ 
                            & $\textrm{0} \textcolor{lightgray}{\pm \textrm{0}}$ 
                            & \colorbox{tableblue}{$\textrm{1} \textcolor{lightgray}{\pm \textrm{1}}$}
                                        \\
\texttt{puzzle-4x5}         & $\textrm{-}$ 
                            & $\textrm{15} \textcolor{lightgray}{\pm \textrm{3}}$ 
                            & $\textrm{16} \textcolor{lightgray}{\pm \textrm{7}}$ 
                            & $\textrm{18} \textcolor{lightgray}{\pm \textrm{1}}$ 
                            & $\textrm{14} \textcolor{lightgray}{\pm \textrm{4}}$ 
                            & \colorbox{tableblue}{$\textrm{20} \textcolor{lightgray}{\pm \textrm{0}}$}
                            & $\textrm{12} \textcolor{lightgray}{\pm \textrm{4}}$ 
                            & $\textrm{13} \textcolor{lightgray}{\pm \textrm{3}}$ 
                            & \colorbox{tableblue}{$\textrm{20} \textcolor{lightgray}{\pm \textrm{0}}$}
                                        \\
\texttt{puzzle-4x6}         & $\textrm{-}$  
                            & $\textrm{4} \textcolor{lightgray}{\pm \textrm{5}}$ 
                            & $\textrm{0} \textcolor{lightgray}{\pm \textrm{0}}$
                            & $\textrm{4} \textcolor{lightgray}{\pm \textrm{9}}$
                            & $\textrm{3} \textcolor{lightgray}{\pm \textrm{5}}$ 
                            & $\textrm{4} \textcolor{lightgray}{\pm \textrm{4}}$ 
                            & $\textrm{3} \textcolor{lightgray}{\pm \textrm{6}}$ 
                            & $\textrm{3} \textcolor{lightgray}{\pm \textrm{3}}$ 
                            & \colorbox{tableblue}{$\textrm{8} \textcolor{lightgray}{\pm \textrm{17}}$}
\\
\cmidrule(lr){2-10}
Overall (\textit{Hard})     & $\textrm{-}$ 
                            & $\textrm{5} \textcolor{lightgray}{\pm \textrm{4}}$ 
                            & $\textrm{4} \textcolor{lightgray}{\pm \textrm{4}}$ 
                            & $\textrm{6} \textcolor{lightgray}{\pm \textrm{5}}$ 
                            & $\textrm{4} \textcolor{lightgray}{\pm \textrm{3}}$ 
                            & $\textrm{7} \textcolor{lightgray}{\pm \textrm{3}}$ 
                            & $\textrm{7} \textcolor{lightgray}{\pm \textrm{6}}$ 
                            & $\textrm{5} \textcolor{lightgray}{\pm \textrm{3}}$ 
                            & \colorbox{tableblue}{$\textrm{12} \textcolor{lightgray}{\pm \textrm{9}}$}
\\
\bottomrule
\end{tabular}}
\end{table}

\begin{figure}[t]
\vspace{-5pt}
\centering
\includegraphics[width=0.62\textwidth]{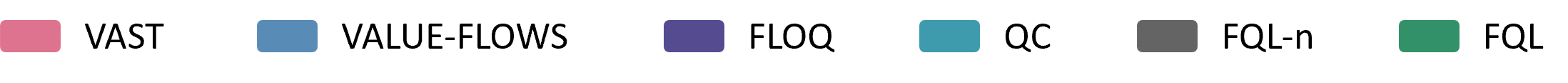}
\includegraphics[width=\textwidth]{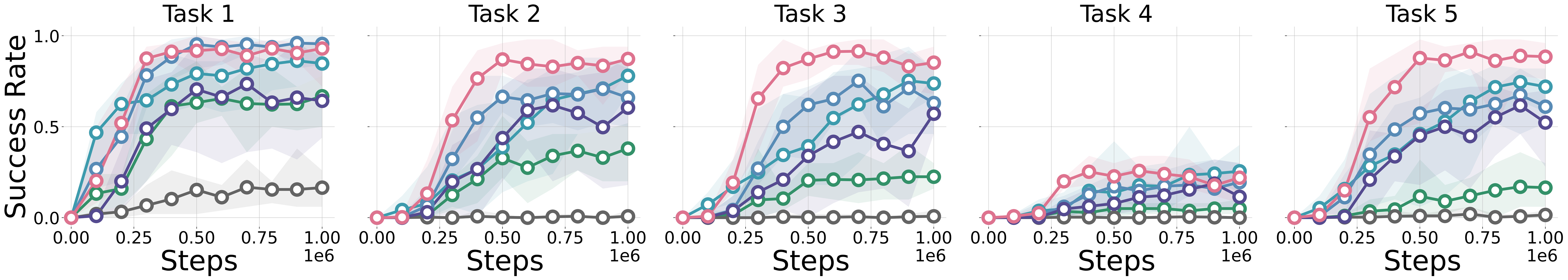}
\includegraphics[width=\textwidth]{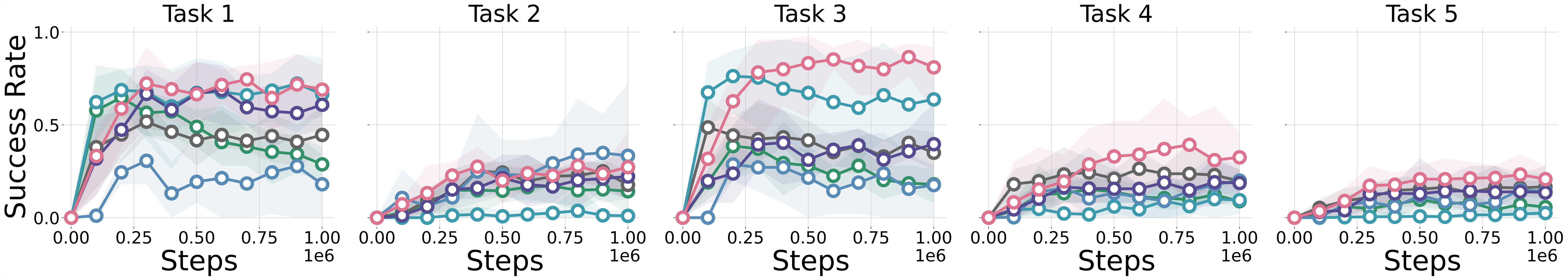}
\caption{\label{fig:curve_main} \small \textbf{Learning curves.} (\textit{Top}): \texttt{cube-double-play}. (\textit{Bottom}): \texttt{puzzle-4x4-play}.}
\vspace{-10pt}
\end{figure}

\begin{figure}[t]
    \includegraphics[width=\linewidth]{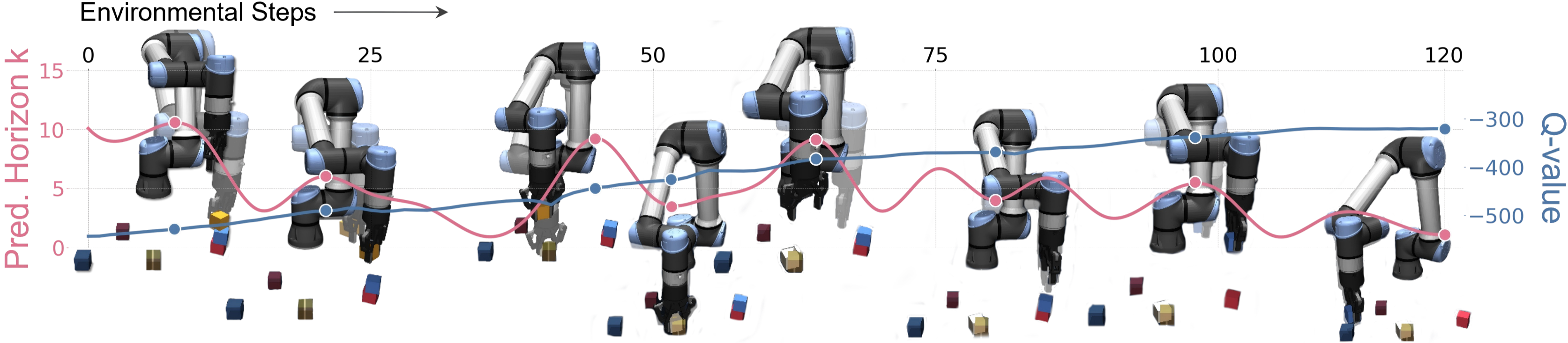}
    \vspace{-10pt}
    \captionof{figure}{\label{fig:visualization} \textbf{Visualization on \texttt{cube-triple-play-task3} (pop\_from\_stuck).} The solid robot arm shows the current pose, while the transparent arm shows the sub-goal.}
\vspace{-5pt}
\end{figure}

\begin{figure}[t]
\centering
\includegraphics[width=0.98\textwidth]{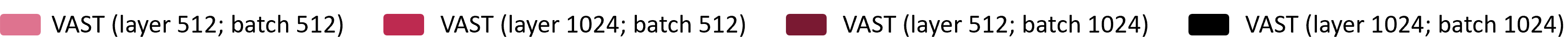}
\includegraphics[width=1\textwidth]{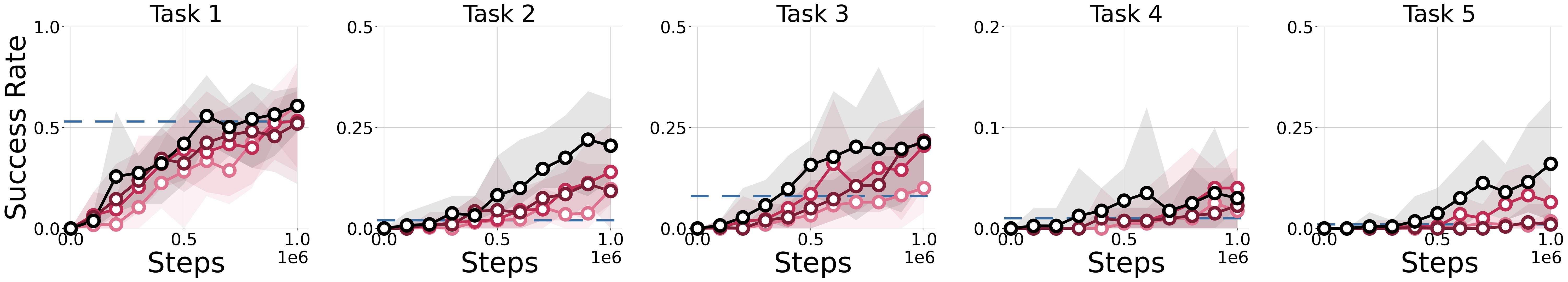}
\caption{\label{fig:scaling} \small \textbf{Performance on \texttt{cube-triple-play} with different layer widths and batch sizes.} The \textcolor[RGB]{59,110,165}{blue dashed line} marks the best baseline on each task. \textit{VAST} benefits from larger layer widths and batch sizes.}
\vspace{-10pt}
\end{figure}

\begin{figure}[b]
    \centering
    \begin{minipage}[c]{0.19\linewidth}
        \includegraphics[width=0.99\linewidth]{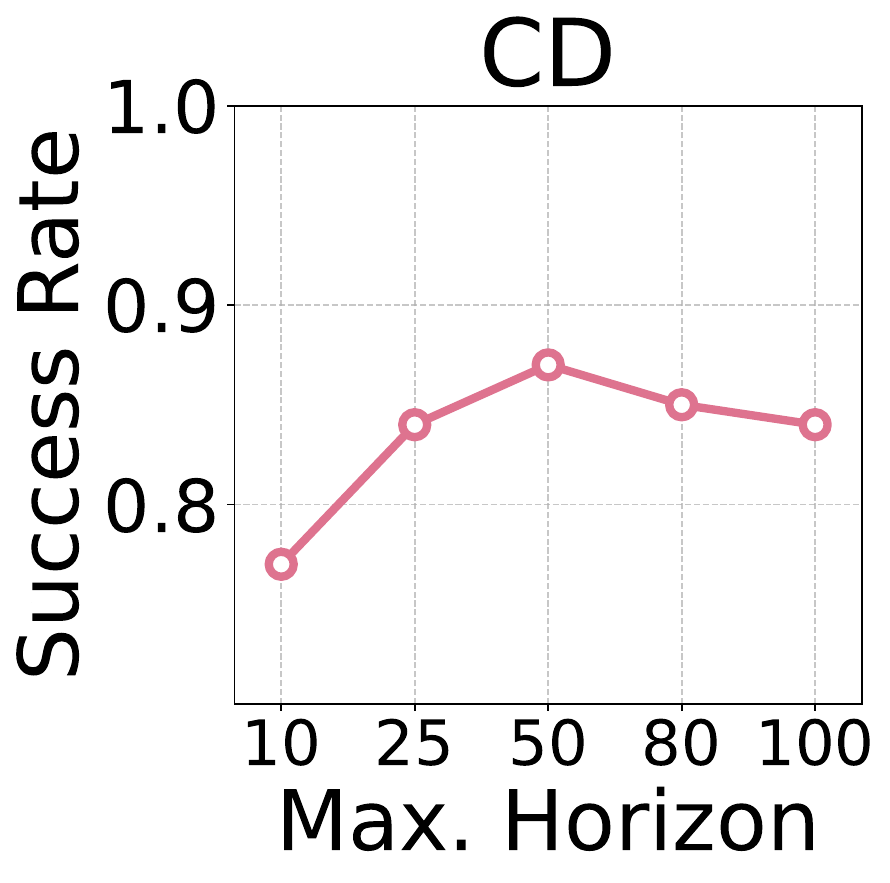}
    \end{minipage}
    \hfill
    \begin{minipage}[c]{0.19\linewidth}
        \includegraphics[width=0.99\linewidth]{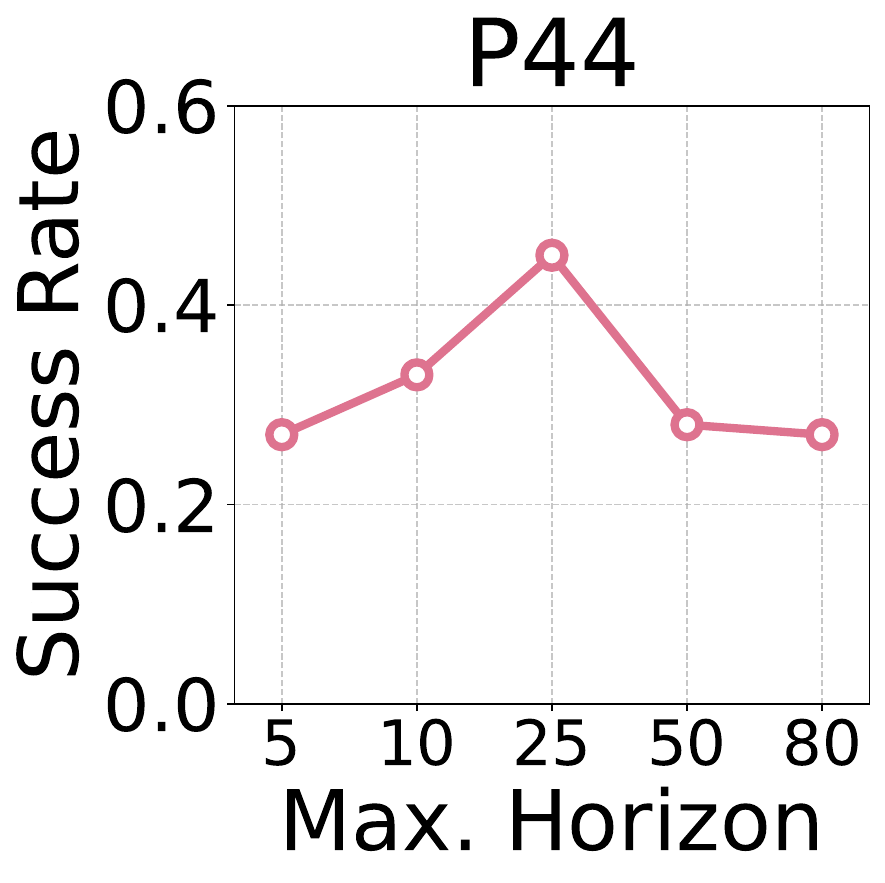}
    \end{minipage}
    \hfill
    \begin{minipage}[c]{0.19\linewidth}
        \includegraphics[width=0.99\linewidth]{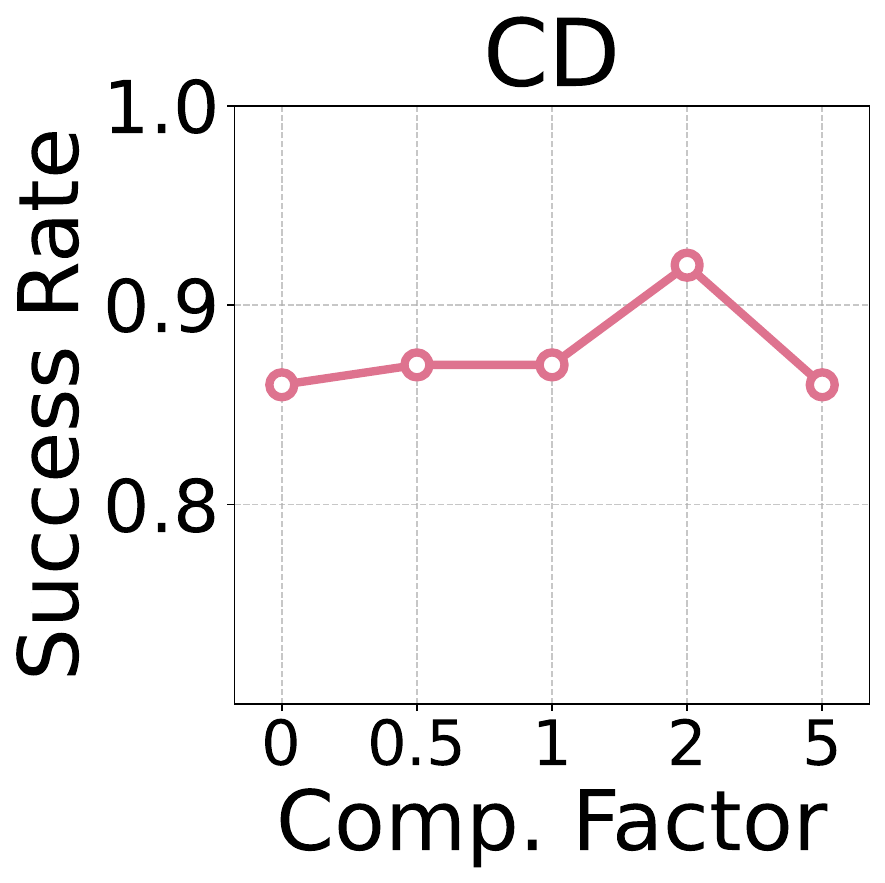}
    \end{minipage}
    \hfill
    \begin{minipage}[c]{0.19\linewidth}
        \includegraphics[width=0.99\linewidth]{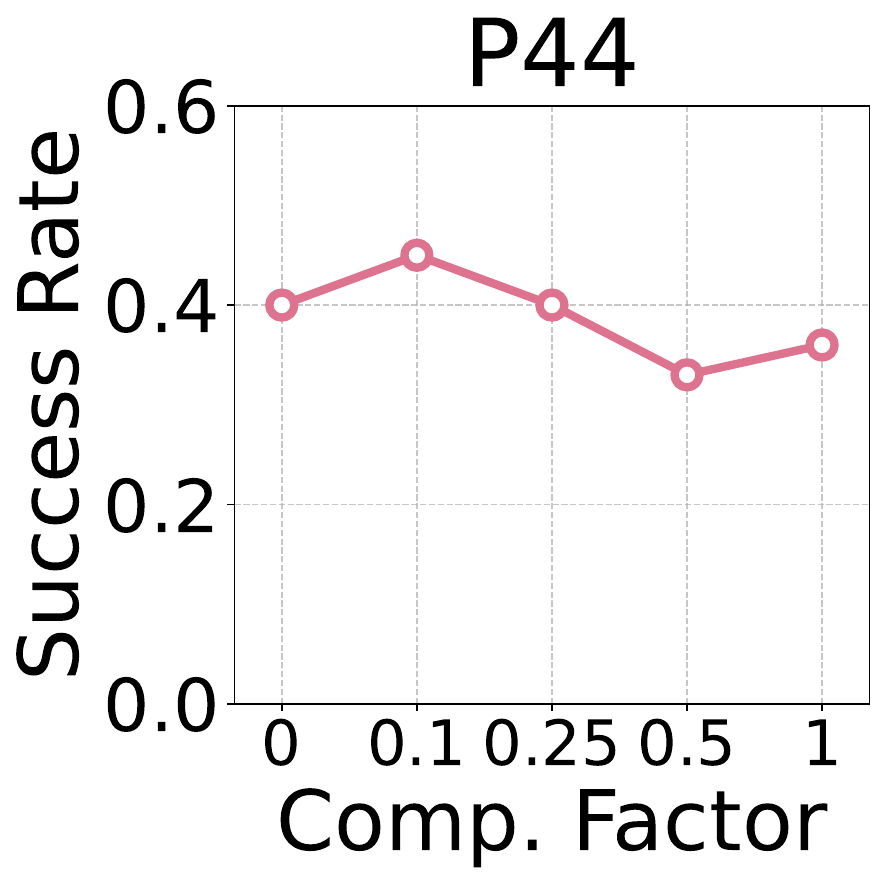}
    \end{minipage}
    \hfill
    \begin{minipage}[c]{0.19\linewidth}
        \includegraphics[width=0.99\linewidth]{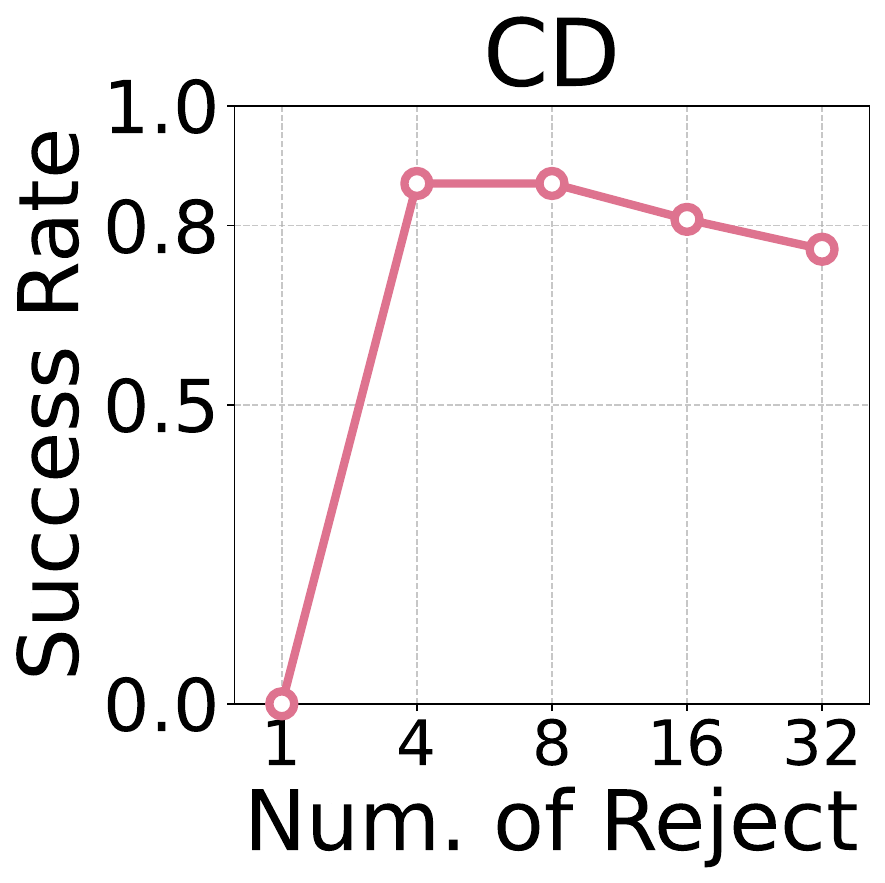}
    \end{minipage}
    \caption{\label{fig:abl} \small \textbf{Ablations on important hyperparameter}. We provide the horizon length ablation (\emph{left 1-2}), the compositional factor ablation (\emph{left 3-4}), and the rejection sampling strength ablation (\emph{left 5}) using default tasks on domain \texttt{cube-double-play} and \texttt{puzzle-4x4-play}.}
\vspace{-10pt}
\end{figure}

\paragraph{Experimental results and ablations.} We evaluate each method with a fixed number of offline updates (1M) and report the success rate. We average the scores across 8 random seeds on main results, with $\pm$ indicating standard deviations in the tables. For each ablation, we average the results on 3 random seeds. We analyze the experiments with the following questions:

\textit{How does VAST perform?}\quad
Table~\ref{table:ogbench_common} summarizes the aggregate results on OGbench. \textit{VAST} achieves the best or near-best performance in both domain-wise averages and the overall score, with particularly pronounced gains on manipulation domains such as \texttt{scene}, \texttt{cube}, and \texttt{puzzle}. Notably, \textit{VAST} is the best or tied-best method on every \textit{Hard} domain, achieving an overall score that is $1.7\times$ that of the best baseline. This demonstrates its ability to effectively exploit dataset transitions and scale to long-horizon decision-making problems. Figure~\ref{fig:curve_main} presents the offline learning curves. \textit{VAST} consistently outperforms competing methods on most tasks, and its sustained upward trend in the \textit{Hard} setting further highlights its scalability. To better understand its behavior, we visualize the ground-truth trajectory together with the future states predicted by \textit{VAST} on \texttt{cube-triple-play-task3} in Figure~\ref{fig:visualization}. The predicted horizon follows a plausible sequence of phased atomic intentions, suggesting that the policy selects an effective intermediate target $k, s_k \sim \pi_\text{stitch}$ and thereby progresses toward stable, value-maximizing actions. Full per-task results and learning curves are provided in Appendix~\ref{appendix:fullresults}.

\textit{Does VAST benefit from wider networks and larger batches?}\quad
Figure~\ref{fig:scaling} reports the performance of \textit{VAST} on \texttt{cube-triple-play} with layer widths and batch sizes of 512 and 1024. Performance generally improves with larger layer widths and batch sizes, and setting both to 1024 yields the best results on most tasks, indicating strong potential for scaling with model capacity and compute. However, this improvement comes at a higher computational cost: increasing both nearly triples the GPU time for 1M training steps (see Appendix~\ref{appendix:fullresults_computation}). For a practical trade-off between performance and efficiency, we report the main results with the default layer width and batch size of 512.

\textit{Does the horizon-length limit affect VAST?}\quad
Figure~\ref{fig:abl} (\emph{left 1-2}) compares the performance of \textit{VAST} under different horizon limits $K$ in value learning (Eq.~(\ref{eql:g_mc}-\ref{eql:g_comp}), Eq.~(\ref{eql:v_def})). Overall, performance improves as $K$ increases from 5, indicating that longer horizons provide greater flexibility for selecting informative targets during learning. Intuitively, a larger $K$ expands the candidate set of future targets and therefore increases the chance of identifying a high-value intermediate target $(k, s_k)$. In most manipulation domains, \textit{VAST} achieves its best performance around $[10,50]$, strongly relying on the task setting. Further increasing $K$ leads to a mild performance drop, suggesting that excessively long horizons may make target estimation more challenging and place higher demands on the expressivity of the $G$-function.

\textit{How important is the compositional objective?}\quad Figure~\ref{fig:abl} (\emph{left 3-4}) reports the performance of \textit{VAST} under different compositional weights $\lambda$ (Eq.~\ref{eql:g_loss})). Performance slightly improves as $\lambda$ increases within an appropriate scale, demonstrating its benefit. While the MC objective (Eq.~(\ref{eql:g_mc})) provides direct supervision from dataset transitions, the compositional objective further encourages temporal consistency of the learned $G$-function and improves its generalization beyond observed trajectories. As a result, it supplies more reliable value estimates for downstream policy optimization. When $\lambda$ becomes too large, performance decreases, suggesting that the compositional term should be balanced with the MC objective; overly strong self-supervised regularization may dominate the learning signal and lead to less informative representations, such as degenerate fixed-point solutions.

\textit{Is two-level reward maximization necessary?}\quad
We compare different rejection sampling strengths for $\pi_\text{exec}$ in Figure~\ref{fig:abl} (\emph{left 5}). When $N_\text{rej}=1$, the framework effectively relies on $\pi_\text{stitch}$ alone, where reward maximization is mainly achieved through goal selection and the goal-conditioned BC policy directly executes the selected goal. Increasing $N_\text{rej}$ improves performance, indicating that additional reward-aware selection at the execution level is a useful complement coupled with high-level goal selection. These results validate the two-level reward-maximization design of \textit{VAST}.
\section{Related Works}
\label{sec:related}

\noindent\textbf{Offline RL.}\quad Offline RL aims to learn an optimal policy from the given dataset without environmental interaction. A central challenge is distributional shift, where value functions may assign overly optimistic estimates to out-of-distribution (OOD) actions, leading to poor policy improvement~\citep{Levine2020OfflineRL}. Early methods mainly address this issue from two perspectives. The first line constrains the learned policy to stay close to the behavior policy~\citep{Fujimoto2018OffPolicyDR,kumar2019stabilizing,fujimoto2021minimalist,li2023proto,li2023when,cheng2023look}. The second line directly regularizes the value function by penalizing overestimated OOD queries~\citep{kummar2020cql,Kostrikov2021fishercritic,niu2022when,yang2022rorl,xu2022constraints}. More recent studies suggest that, beyond conservative value estimation, effective policy extraction is often the key factor that determines final performance~\citep{park2024valuelearningreallymain}. This insight has motivated lightweight in-sample learning methods~\citep{xu2022a,Kostrikov2021OfflineRL,Xu2023OfflineRW,wang2023offline}, which are stable, simple, and competitive in absolute performance. In parallel, generative models such as diffusion~\citep{diffusion2015Sohl,ho2020DDPM,score2023songyang} and flow-based models~\citep{lipman2024flowmatchingguidecode} have been introduced to improve policy expressivity and generalization~\citep{zheng2024safe,liang2026dipole,zheng2025towards,Mao2024DiffusionDICEID,Wang2022DiffusionPA,idql,Chen2022OfflineRL,CEP,frans2025diffusionguidancecontrollablepolicy,espinosa2025scaling,li2026qam}. Together, these developments have made offline RL increasingly practical for real-world applications~\citep{intelligence2025pi06vlalearnsexperience,zheng2026unleash}.

\noindent\textbf{Value learning in RL.}\quad
Value learning is central to RL, as it determines the performance ceiling of the learned policy; meanwhile, it remains a key bottleneck for scaling RL to complex and long-horizon problems~\citep{park2025horizonreductionmakesrl}. Existing efforts to improve value learning can be broadly grouped into three directions: (i) stabilizing and strengthening value estimation through multi-step reward propagation, which includes multi-step TD backups in the action space~\citep{park2025horizonreductionmakesrl,konidaris2011td_gamma}, as well as methods that couple value estimates over fixed-length action or trajectory chunks~\citep{seo2024coarse,tian2025chunking,li2025reinforcement,li2025decoupledqchunking}; (ii) distributional approaches model the full return distribution rather than only its expectation, either explicitly learn a discrete return distribution with projection operators~\citep{bellemare2017distributional,Stopfarebrother24a,ma2025dsac}, or implicitly represent return distributions using expressive generative models~\citep{agrawalla2025floq,dongzheng2026value}, including diffusion~\citep{diffusion2015Sohl,ho2020DDPM,score2023songyang} and flow matching~\citep{lipman2023flow,lipman2024flowmatchingguidecode} models; (iii) scaling the parameterization of value functions with expressive architectures, such as Transformers and Mixtures of Experts (MoEs)~\citep{chebotar2023q,pac-springenberg24a,ceron24b}.

\noindent\textbf{Hierarchical and option in RL.}\quad
The option framework was introduced to improve the scalability of RL by abstracting temporal dynamics into reusable high-level behaviors~\citep{SUTTON1999181,precup2000temporal,stolle2002learning,bacon2016optioncriticarchitecture,riemer2019learningabstractoptions}. In this framework, an option is typically defined by an initiation condition, a termination condition, and an intra-option policy. This formulation enables agents to reason and act over extended time scales, and has naturally led to a broad line of hierarchical RL (HRL) methods~\citep{park2025horizonreductionmakesrl,machado2023temporal,stolle2002learning,bacon2016optioncriticarchitecture,riemer2019learningabstractoptions,xu2022a,hiql,vezhnevets2017feudal,dayan1992feudal,dietterich2000hierarchical,vezhnevets2016strategic,pateria2021hierarchical,kulkarni2016hierarchical}. Despite this progress, HRL still faces a key optimization challenge lies in non-stationary learning, where the high-level policy must optimize over goals or options whose effects keep changing as the low-level policy is updated~\citep{nachum2018data}.
\section{Conclusion and Future}
\label{sec:conclusion}

In this work, we introduced \textit{VAST}, a novel RL method that decouples value learning from fixed-step backups via an auxiliary value function that enables "stitching" operation. \textit{VAST} defines a stitching Bellman operator and employs two-level policies to maximize reward. Its excellent performance on OGBench suggests a promising direction for scalable long-horizon decision-making with horizon-aware future identification. Looking forward, the core of \textit{VAST} could be naturally combined with several diffusion/flow policy optimization methods, model-based methods, and representation learning methods, offering an exciting avenue for further improvement of horizon-adaptive RL. See Appendix~\ref{appendix:limitation_future} for more discussion on limitations and future work.

\section*{Acknowledgment}
This work was supported by the Wuxi Research Institute of Applied Technologies at Tsinghua University under the Grant 20242001120, the Xiongan AI Institute, the Early Career Scheme (ECS) of the Research Grants Council of Hong Kong under Project No. HKU27206826, the University of Hong Kong Faculty Interdisciplinary Fund, and the University of Hong Kong Urban Systems Institute (HKU-USI) Fellowship Grant. We would also like to thank the NeurIPS review committee for their helpful comments and suggestions on this work.

\bibliography{reference}
\bibliographystyle{unsrtnat}

\newpage
\appendix
\raggedbottom
\section{LLM Usage}
\label{appendix:llm}

In this paper, we employed Large Language Models (LLMs) solely for polishing the writing. No parts of the technical content, experimental results, or conclusions were generated by LLMs.
\section{Theoretical Proofs}
\label{appendix:proof}

\subsection{Proof of Theorem \ref{thm:compositional} (Temporal Compositionality)}
\label{appendix:proof_compostional}

\textbf{Proof.}
By definition, $G(s, s_k, k) = \mathbb{E}_{\pi}\left[ \sum_{t=0}^{k-1} \gamma^t r_t \mid s_0 = s \right]$, where $k$ denotes the horizon between states $s$ and $s'$ .
For any intermediate state $s_i$ reached at time $i < k$, we can partition the summation:
\begin{align}
G(s, s_k, k) &= \mathbb{E}_{\pi}\left[ \sum_{t=0}^{i-1} \gamma^t r_t + \sum_{t=i}^{k-1} \gamma^t r_t \;\middle|\; s_0 = s,(s_k, k) \right] \\
&= \mathbb{E}_{i \sim \mathrm{Unif}\{1,...,k-1\}, s_i \sim \pi} \left( \mathbb{E}_{\pi}\left[ \sum_{t=0}^{i-1} \gamma^t r_t \;\middle|\; s_0=s, (s_i, i) \right] \right. \\
&\qquad\qquad\qquad \left. + \, \mathbb{E}_{\pi}\left[ \gamma^{i} \sum_{t=0}^{k-i-1} \gamma^t r_{t+i} \;\middle|\; s_{0} = s_i, (s_k, k-i) \right] \right) \\
&= \mathbb{E}_{i \sim \mathrm{Unif}\{1,...,k-1\}, s_i \sim \pi} \left[G(s, s_i, i) + \gamma^{i} G(s_i, s_k, k-i)\right],
\end{align}
where $\pi$ denotes the behavioral distribution.
Under the function approximation $\hat{G}$, let the error at each compositional step be bounded by $\epsilon_{\text{comp}}$. For a trajectory composed of $M$ segments $(s_0, s_1, \dots, s_M)$, the approximated return is equivalent to $\sum_{t=0}^{M-1} \gamma^{t} \hat{G}(s_t, s_{t+1}, 1)$, which implies the accumulated error is at most $\sum_{t=0}^{M-1} \gamma^{t} \epsilon_{\text{comp}} \le M \epsilon_{\text{comp}}$, under a linear growth w.r.t segments. \hfill $\blacksquare$

\subsection{Proof of Theorem \ref{thm:contraction} (Contraction Mapping)}
\label{appendix:proof_contraction}
\textbf{Proof.}
We evaluate the supremum norm between the operator applied to two arbitrary bounded value functions $V_1$ and $V_2$:
\begin{equation}
\begin{aligned}
\|\mathcal{T}_{\text{stitch}} V_1 - \mathcal{T}_{\text{stitch}} V_2\|_\infty &= \max_{s} \left| \max_{k_1 \in \{1,...,K\}} \big[ G(s, s_{k_1}, k_1) + \gamma^{k_1} V_1(s_{k_1}) \big] \right. \\
&\qquad\qquad\qquad\qquad\qquad\qquad \left. - \max_{k_2 \in \{1,...,K\}} \big[ G(s, s_{k_2}, k_2) + \gamma^{k_2} V_2(s_{k_2}) \big] \right| \\
&\le \max_{s} \left| \max_{k_1 \in \{1,...,K\}} \big[ G(s, s_{k_1}, k_1) + \gamma^{k_1} V_1(s_{k_1}) \big] -  \big( G(s, s_{k_1}, k_1) + \gamma^{k_1} V_2(s_{k_1}) \big) \right| \\
&\le \max_{s} \max_{k \in \{1,...,K\}} \left| \big( G(s, s_k, k) + \gamma^k V_1(s_k) \big) - \big( G(s, s_k, k) + \gamma^k V_2(s_k) \big) \right| \\
&= \max_{s} \max_{k \in \{1,...,K\}} \gamma^k \left| V_1(s_k) - V_2(s_k) \right| \\
&\le \gamma^{\tilde{k}} \| V_1 - V_2 \|_\infty \\
&\leq \gamma \| V_1 - V_2 \|_\infty.
\end{aligned}
\end{equation}

The last inequality holds as any valid temporal transition requires at least one environment step, $\tilde{k} \ge 1$, which implies $\gamma^{\tilde{k}} \le \gamma < 1$, proving that $\mathcal{T}_{\text{stitch}}$ is a strict $\gamma$-contraction. By the Banach Fixed-Point Theorem, repeated application of $\mathcal{T}_{\text{stitch}}$ converges to a unique fixed point $V^*$. \hfill $\blacksquare$

\section{Details on the Toy Example}
\label{appendix:toy_details}

The values visualized in Figure~\ref{fig:toycase} are computed \emph{analytically}: for each value-learning scheme, we apply exact dynamic programming over the transitions contained in the dataset (the green and purple paths), proceeding recursively from the goal cell at the top-right corner back to the start cell at the bottom-left corner. We use $\gamma=0.99$ and initialize the value of the goal cell to $V((3,3))=0$. For example, under one-step TD, $Q((2,3),\text{right}) = 1 + \gamma\cdot V((3,3)) = 1$. Under this setup, there exist unique optimal path passes through $(1,0)$. The key failure mode of fixed-horizon modeling is caused by the negative reward at $(1,2)$ that distorts the value of nearby states along the purple path, which further avoid the optimal transition $(0,0)\to(1,0)$. Figure~\ref{fig:toycase} shades only the cells that are critical for exposing this failure mode.

\section{Experimental Setup}
\label{appendix:experimental}

\subsection{Environment}
\label{appendix:experimental_env}

We evaluate \textit{VAST} on the OGBench task suite~\citep{ogbench_park2025}. OGBench is a benchmark for goal-conditioned offline RL, with challenging robotic locomotion and manipulation tasks, including whole-body humanoid control, maze navigation, and object manipulation. Our results cover 30 tasks across 6 widely used domains (\texttt{scene-play, cube-single-play, cube-double-play, puzzle-4x4-play, antmaze-large-navigate, antmaze-giant-navigate}), denoted as \textit{Basic} (refer to Figure~\ref{fig:ogbench_env_common}), and 20 tasks across 4 highly challenging domains (\texttt{cube-triple-play, cube-quadruple-play, puzzle-4x5-play, puzzle-4x6-play}), denoted as \textit{Hard} (refer to Figure~\ref{fig:ogbench_env_challenge}), for a total of 50 tasks on OGBench in \texttt{singletask} variant for offline RL. All tasks we evaluated are state-based; evaluation is performed on five pre-defined state-goal pairs. We introduce their settings in detail as follows:

\begin{figure}[H]
    \centering
    \includegraphics[width=0.80\linewidth]{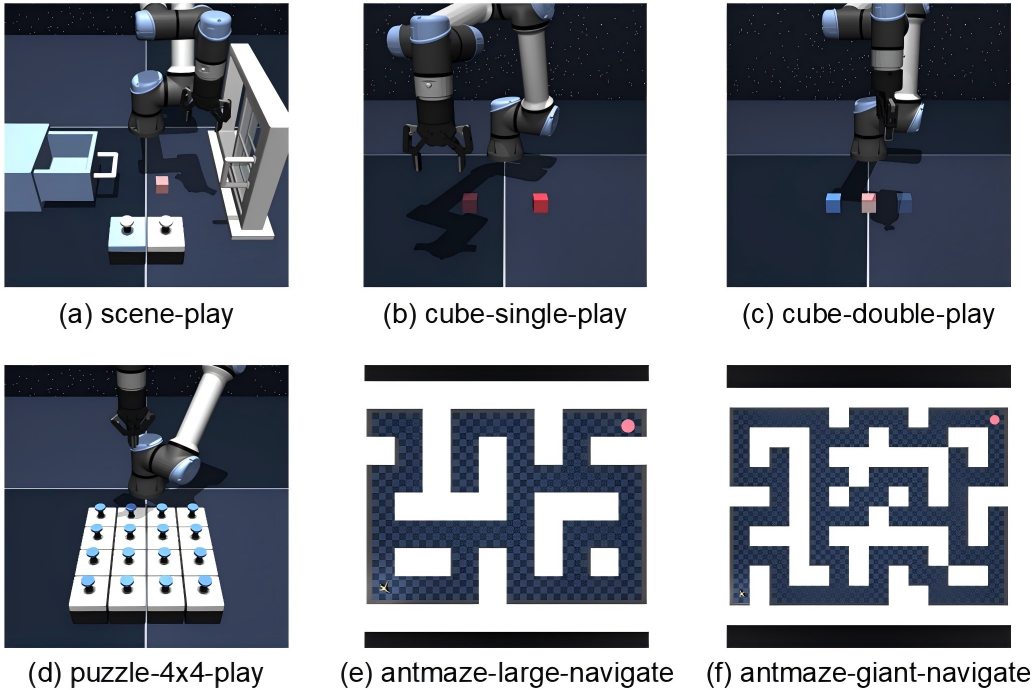}
    \caption{\textbf{OGBench Basic.} We experiment on 6 majorly used domains on OGBench~\citep{ogbench_park2025} in singletask variant, shown as (a)-(f). Each domain contains 5 tasks.}
    \label{fig:ogbench_env_common}
    \vspace{0.5em}
    \includegraphics[width=1\linewidth]{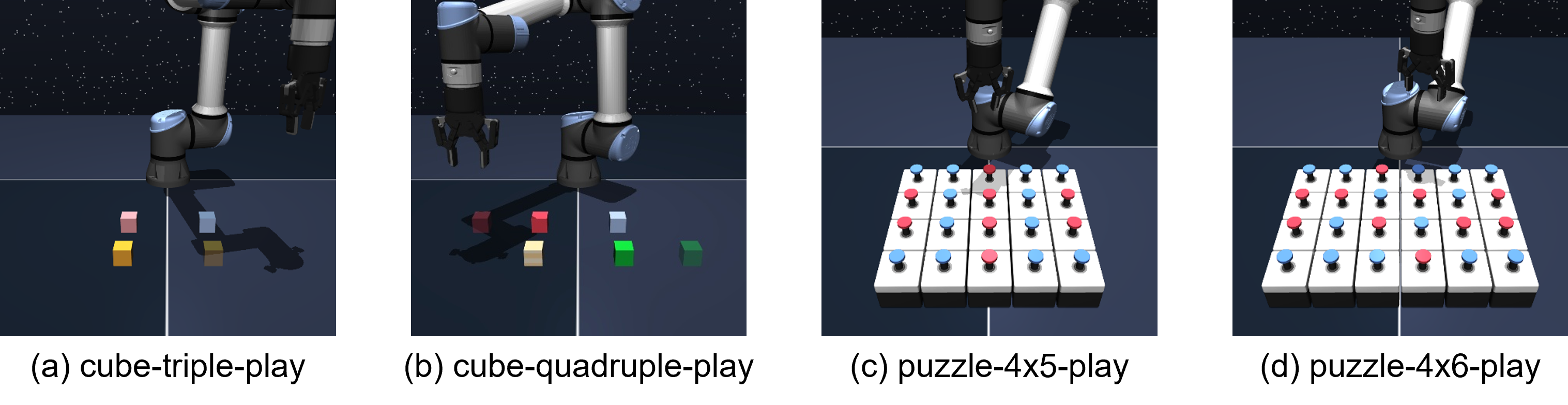}
    \caption{\textbf{OGBench Hard.} We experiment on 4 highly challenging domains on OGBench~\citep{ogbench_park2025} in singletask variant, shown as (a)-(d). Each domain contains 5 tasks.}
    \label{fig:ogbench_env_challenge}
\end{figure}

\paragraph{Manipulation.} Manipulation tasks are performed with the UR5e robot arm and Robotiq 2F-85 gripper models from MuJoCo Menagerie~\citep{menagerie2022github} and the drawer, window, and button box models from Meta-World~\citep{yu2020meta}. We utilize \texttt{play} datasets for all domains, collected by open-loop, non-Markovian scripted policies defined by~\citet{ogbench_park2025}.

\begin{itemize}[leftmargin=1.4em, itemsep=1pt, parsep=0pt, topsep=2pt]
    \item \textbf{scene}: Long-horizon manipulation involves interacting with multiple object types within one environment, including a cube block, a window, a drawer, and two button locks, where actions change object states and enable subsequent manipulation.

    \item \textbf{cube}: Pick-and-place manipulation of cube-shaped blocks to form specified target configurations. We use \texttt{single} to move one cube and \texttt{double} to arrange two cubes in the shared domain for common evaluation. We further select \texttt{triple} and \texttt{quadruple}, which place three and four cubes, respectively, as the challenge domains, thereby strongly increasing the combinatorial difficulty of assembling target configurations.
    
    \item \textbf{puzzle}: Solving with a robot arm where pressing buttons induces deterministic state transitions according to the puzzle rule (e.g., toggling the pressed cell and its neighbors). This series of domains is designed to test the agent's combinatorial generalization abilities. We use \texttt{4x4} ($4\times4$ grid) for common evaluation. We further select hard settings with larger grid sizes, \texttt{4x5} ($4\times5$ grid) and \texttt{4x6} ($4\times6$ grid), as the challenging domains.
    
\end{itemize}

\paragraph{Navigation.} Environments are adopted from Mujoco~\citep{todorov2012mujoco}, and the success of this series of domains is determined by the achieved proximity to a specified goal location in the maze. We utilize \texttt{navigate} datasets for all domains, collected by a noisy expert policy that navigates the maze by repeatedly reaching randomly sampled goals.

\begin{itemize}[leftmargin=1.4em, itemsep=1pt, parsep=0pt, topsep=2pt]
    \item \textbf{antmaze}: Controlling a quadrupedal Ant agent with 8 degrees of freedom (DoF) to reach a specified goal location, with success evaluated by the achieved proximity to the goal in the maze. We select \texttt{large} maze and \texttt{giant} maze (twice the size of \texttt{large}) in this work.
\end{itemize}

\subsection{Evaluation}
\label{appendix:experimental_evaluation}

We evaluate all methods trained on standard datasets provided by OGBench, and report the success rate at 1M gradient steps for each task. Each evaluation consists of 50 episodes. The data specifications are shown in Table~\ref{table:ogbench_data}.

\begin{table}[H]
\centering
\caption{\label{table:ogbench_data} \small \textbf{Data specifications.} Reported from~\citet{ogbench_park2025}}
\small
\setlength{\tabcolsep}{3.5pt}
\begin{tabular}{lccccc}
\toprule
\textbf{Dataset} & Transitions & Episodes & Data Episode Length & State Dim. & Action Dim. \\
\midrule
\texttt{scene-play}  & 1M    & 1000  & 1000  & 40    & 5
\\
\texttt{cube-single-play}    & 1M    & 1000  & 1000  & 28    & 5
\\
\texttt{cube-double-play}    & 1M    & 1000  & 1000  & 37    & 5
\\
\texttt{puzzle-4x4-play} & 1M    & 1000  & 1000  & 83    & 5
\\
\texttt{antmaze-large-navigate}   & 1M   & 1000  & 1000  & 29    & 8
\\
\texttt{antmaze-giant-navigate}   & 1M   & 500   & 2000  & 29    & 8
\\
\cmidrule(lr){1-6}
\texttt{cube-triple-play}    & 3M    & 3000  & 1000  & 46    & 5
\\
\texttt{cube-quadruple-play} & 5M    & 5000  & 1000  & 55    & 5
\\
\texttt{puzzle-4x5-play} & 3M    & 3000  & 1000  & 99    & 5
\\
\texttt{puzzle-4x6-play} & 5M    & 5000  & 1000  & 115    & 5
\\
\bottomrule
\end{tabular}
\end{table}

\subsection{Baselines}
\label{appendix:experimental_baselines}

In this section, we provide the details of the baselines in this work. We select \textit{IQL}, \textit{FQL}, \textit{IFQL}, \textit{FQL-n}, \textit{IFQL-n}, \textit{QC}, \textit{VALUE-FLOWS}, and \textit{FLOQ}, in total of 8 baselines for a comprehensive comparison. Specifically, we report the performance of \textit{IQL} from~\citet{fql_park2025} and use the open-source codebases for experiments on others, using the corresponding default hyperparameters.

\begin{itemize}[leftmargin=1.2em, itemsep=1pt, parsep=0pt, topsep=2pt]
\item \textbf{\textit{IQL~\citep{Kostrikov2021OfflineRL}.}} \textit{IQL} is a strong and stable Gaussian-policy offline RL baseline. The key idea of \textit{IQL} is to learn a state value function via upper expectile regression over in-dataset action values,  and then use this value estimate to bootstrap the $Q$-function. The policy is subsequently extracted via \textit{AWR}-style~\citep{peng2019advantage} advantage-weighted behavioral cloning, thereby enabling policy improvement while remaining close to the data distribution.

\item \textbf{\textit{FQL/FQL-n~\citep{fql_park2025}.}} \textit{FQL} is a strong flow-policy offline RL baseline built on a behavior-regularized actor-critic framework. The key idea of \textit{FQL} is to train a behavior-cloned flow model for expressive action modeling, while learning a separate one-step policy that maximizes the critic under a distillation regularizer. Specifically, \textit{FQL} avoids unstable backpropagation through time and eliminates costly iterative inference, while still retaining much of the expressivity of flow-based policies by decoupling value maximization from the iterative flow-generation process. \textit{FQL-n} is the $n$-step backup variant of \textit{FQL}.

\item \textbf{\textit{IFQL/IFQL-n~\citep{fql_park2025}.}} \textit{IFQL} is a flow-based variant of \textit{IQL}~\citep{Kostrikov2021OfflineRL} that combines expectile-regression value learning with an imitation-trained flow policy and rejection sampling for policy extraction. \textit{IFQL}, therefore, is a traditional, strong, and stable in-sample baseline that benefits from expressive flow-based action modeling while avoiding direct optimization through the generative process. \textit{IFQL-n} is the $n$-step backup variant of \textit{IFQL}.

\item \textbf{\textit{QC~\citep{li2025reinforcement}.}} \textit{QC} is a strong baseline designed for long-horizon sparse-reward problems via action chunking. \textit{QC} learns a flow-matching behavior policy over temporally extended action chunks and a chunk-level critic, and performs rejection sampling to choose the action sequence with the highest predicted value. This chunked formulation enables longer-horizon TD back-ups and more temporally coherent exploration, which substantially improves learning in tasks where single-step control is ineffective. It is a representative baseline that explicitly addresses long-horizon difficulty through temporal abstraction and chunk-level value learning.

\item \textbf{\textit{VALUE-FLOWS~\citep{dongzheng2026value}.}} \textit{VALUE-FLOWS} is a strong distributional RL baseline that models the full return distribution using flow matching, rather than collapsing future outcomes into a single scalar $Q$-value. It learns a return vector field that satisfies the distributional Bellman equation and additionally uses estimated return variance to reweight training toward more uncertain transitions. \textit{VALUE-FLOWS} is particularly appealing for long-horizon tasks with multimodal or uncertain returns.

\item \textbf{\textit{FLOQ~\citep{agrawalla2025floq}.}} \textit{FLOQ} is a strong value-learning baseline that scales value estimation through iterative computation. Specifically, \textit{FLOQ} represents the $Q$-function as a flow-matching velocity field that transforms noise into a scalar $Q$-value through numerical integration, with dense supervision applied across intermediate integration steps. In practice, \textit{FLOQ} is instantiated on top of \textit{FQL}, preserving the expressive flow-policy framework while substantially improving critic capacity and enabling a form of test-time compute scaling through additional integration steps. It is a particularly strong baseline for evaluating improvements from value scaling.

\end{itemize}
\section{Implementation}
\label{appendix:implementation}

This section provides additional implementation information. See pseudo-code in Section~\ref{appendix:algo}, implementation and structural details in Section~\ref{appendix:imp_details}, and hyperparameters in Section~\ref{appendix:implementation_hyp}.

\subsection{Implementation details}
\label{appendix:imp_details}

We implement \textit{VAST} in JAX~\citep{jax2018github} on top of recent study by~\citet{park2025horizonreductionmakesrl}.

\paragraph{Architectures.} We train 5 neural networks in parallel: 2 of flow policy networks (stitching policy $\pi_\text{stitch}$; execute policy $\pi_\text{exec}$), and 3 of value networks (state value $V$; state-action value $Q$; horizon-based cumulative return $G$). We leverage the multi-layer perceptron (MLP) for each network; the depth and width are held constant within each domain.

\paragraph{Value learning.} In all experiments, we learn the optimal state-value function $V$ using IQL-style expectile regression~\citep{Kostrikov2021OfflineRL}. The sampled bootstrap target $(k, s_k) \sim \pi_\text{stitch}(\cdot \mid s)$ could enable a greedier and policy-aware value update. To better align value learning with both the \emph{stitching} and \emph{execution} policies, we approximate $V_{\text{target}}^{\pi}$ in Eq.~(\ref{eql:v_loss}) as
\begin{equation}
\label{eql:v_target}
V_{\text{target}}^{\pi}=G(s, s_k, k)+\gamma^k Q\bigl(s_k,\pi_\text{exec}(a_k \mid s_k, s_k^+)\bigr),
\end{equation}
where $(k^+, s_k^+) \sim \pi_\text{stitch}(\cdot \mid s_k)$. The sampled horizon $k$ is the continuous output of the flow policy $\pi_\text{stitch}$ and enters $G$ and $\gamma^k$ directly. Here, $\pi_\text{exec}$ is trained to produce actions within the data distribution. As a result, the second term provides a stable approximation to the continuation value at $s_k$, while still reflecting the behavior induced by the combined stitching-and-execution policy.

For learning the $G$-function, given a batch of states $\{s_t\}$, we first uniformly sample a future state $s_{t+k}$ from the same trajectory, where the horizon $k$ is sampled from $\{1,\ldots,\min(K,t_\text{terminal}-t)\}$ with $t_\text{terminal}$ the terminal step of the trajectory, and $K$ denotes the maximum horizon. We then uniformly sample an intermediate state $s_{t+i}$ between $s_t$ and $s_{t+k}$, with $i \sim \{1,\ldots,k-1\}$; only samples with $k\geq 2$ enter the compositional objective. The $G$-function is trained with the Monte Carlo objective in Eq.~(\ref{eql:g_mc}) and the compositional objective in Eq.~(\ref{eql:g_comp}), using a fixed balance weight of $0.5$. The maximum horizon $K$ controls the temporal range of MC supervision for the $G$-function. In practice, choosing $K$ based on the domain helps balance the expressivity of the $G$-network against task complexity. Detailed domain-specific settings are provided in Table~\ref{table:domain_spec_hyp}.

\paragraph{Policy extraction.} We use weighted regression for the objective of \emph{stitching} policy, which has been fully studied in previous works~\citep{zheng2024safe,zheng2025towards,liang2026dipole}. It starts with a regularized optimization problem:
\begin{equation} \label{eql:problem}
\max_{\pi_\text{stitch}} ~\mathbb{E}_{s \sim \mathcal{D}}\bigl[\mathbb{E}_{\pi_\text{stitch}}\left[V_{\text{target}}^{\pi} - V(s) \right]- \frac{1}{\beta} D_\text{KL}(\pi_\text{stitch}\left(\cdot | s\right) \| \mu \left(\cdot | s\right))\bigr],
\end{equation}
where $\mu$ denotes the data distribution. Following the derivation of~\citet{zheng2024safe}, we could have the closed-form solution:
\begin{equation} \label{eql:closed_form}
    \pi^{\star}_\text{stitch}(k, s_k\mid s) \propto \mu(k, s_k\mid s)\cdot \exp \left(\beta \cdot \left(V_{\text{target}}^{\mathcal{D}} - V(s)\right) \right),
\end{equation}
where $V_{\text{target}}^{\mathcal{D}}=G(s_t,s_{t+k},k)+\gamma^{k}Q(s_{t+k},a_{t+k})$, with $(k,s_{t+k},a_{t+k}) \sim \mathcal{D}$. The above form gives us the weighted objectives for the stitching flow networks as Eq.~(\ref{eql:weighted_high}), follow the formulation by~\citet{zheng2024safe}. Given the parameterized execute flow network $\epsilon_\text{execute}$, the objective is:
\begin{equation} \label{eql:bc_low}
\mathcal{L}_{\epsilon_\text{execute}}=\mathbb{E}_{t\sim U[0,1],\epsilon\sim\mathcal{N}(\mathbf{0},\mathbf{I}),(s,a,s_\text{goal})\sim \mathcal{D} }\left[\left\Arrowvert \epsilon-\epsilon_\text{execute}\left(a^t,s,s_\text{goal},t\right) \right \Arrowvert^2 \right].
\end{equation}
At inference time, we use rejection sampling to select the output of the execute policy $\pi_\text{exec}$ with $N_\text{rej}$ variants by tasks. See Section~\ref{appendix:implementation_hyp} for detailed hyperparameters.

\paragraph{Goal sampling mechanism for $\pi_\text{exec}$.}
$\pi_\text{exec}$ is trained in a goal-conditioned behavior-cloning (BC) manner, following the setting of~\citet{park2025horizonreductionmakesrl}. We denote the state-action-goal sampling distribution as $p_\mathcal{D}$, while $p_\mathcal{D}(g\mid s,a)$ is the distribution from which the conditioning goal $s_\text{goal}$ in Eq.~(\ref{eql:bc_low}) is drawn. $p_\mathcal{D}$ is a mixture of three components: $p^\mathcal{D}_\text{cur}$, the Dirac delta at the current state $s$; $p^\mathcal{D}_\text{traj}$, the uniform distribution over the future states of the same trajectory; and $p^\mathcal{D}_\text{rand}$, the uniform distribution over all dataset states. The ratios of these distributions are tunable hyperparameters that sum to 1, refers to \textit{subgoal sample probability} shown in Table~\ref{table:general_hyp}. We use \emph{uniform} sampling over all future states up to the end of the trajectory.

\subsection{Hyperparameters}
\label{appendix:implementation_hyp}

In this section, we provide the detailed hyperparameter setup in Table~\ref{table:general_hyp} for the general hyperparameters, and Table~\ref{table:domain_spec_hyp} for the domain-specific hyperparameters.

\begin{table}[H]
\caption{\textbf{General hyperparameters}.}
\label{table:general_hyp}
\centering
\renewcommand{\arraystretch}{1.05}
\small
\setlength{\tabcolsep}{4pt}
\begin{tabular}{@{}lll@{}}
\toprule
\multicolumn{1}{c}{\multirow{12}{*}{\centering General}}    
        & Hyperparameter                & 
\\ 
\cmidrule(lr){2-3}
        & Optimizer                     & Adam
\\
        & Learning rate                 & 3e-4
\\
        & Offline learning steps        & 1,000,000 
\\
        & Mini-batch                    & 512
\\
        & Soft update factor            & 0.005
\\
        & Flow steps $T$                & 10
\\
        & Gradient clip                 & 1
\\
        & Expectile factor $\tau$       & 0.9
\\
        & Clip Q                        & true (\texttt{antmaze-giant}), false (others)
\\
        & Subgoal sample probability $[p^\mathcal{D}_\text{cur}, p^\mathcal{D}_\text{traj}, p^\mathcal{D}_\text{rand}]$
        & [0,1,0] (\texttt{antmaze-large/giant}),\\[-3pt]& &[0.1, 0.8, 0.1] (others)
\\
\cmidrule(lr){1-3}
\multicolumn{1}{c}{\multirow{3}{*}{\centering Architecture}}        
        & MLP hidden dimension         & [512, 512, 512, 512]
\\
        & Activation function           & GELU~\citep{hendrycks2016gaussian}
\\
        & Layernorm                     & true
\\
\bottomrule
\end{tabular}
\end{table}

\begin{table}[H]
\centering
\caption{\label{table:domain_spec_hyp} \small \textbf{Domain-specific hyperparameters.}}
\small
\setlength{\tabcolsep}{8pt}
\begin{tabular}{lccccc}
\toprule
\textbf{Domain} & $K$ & $\beta$ & $N_\text{rej}$ & $\gamma$ & $\lambda$ \\
\midrule
\texttt{scene-play}         & 50    & 0.03  & 4    & 0.995 & 0.25
\\
\texttt{cube-single-play}   & 50    & 0.08  & 4     & 0.99 & 0.5
\\
\texttt{cube-double-play}   & 50    & 0.08  & 4     & 0.99 & 0.5
\\
\texttt{puzzle-4x4-play}    & 10    & 0.01  & 4 (task 2,5); 16 (others)   & 0.995 & 0.5
\\
\texttt{antmaze-large-navigate} & 50    & 0.05   & 4     & 0.995 & 0.5
\\
\texttt{antmaze-giant-navigate}   & 50    & 0.05   & 4     & 0.995 & 0.5
\\
\cmidrule(lr){1-6}
\texttt{cube-triple-play}    & 25   & 0.001  & 4     & 0.999 & 0.5
\\
\texttt{cube-quadruple-play} & 10   & 0.001  & 16     & 0.999 & 0.1
\\
\texttt{puzzle-4x5-play} & 10    & 0.005  & 32     & 0.995 & 0.5
\\
\texttt{puzzle-4x6-play} & 10    & 0.005  & 32     & 0.995 & 0.5
\\
\bottomrule
\end{tabular}
\end{table}

\subsection{Algorithm Pseudo-code}
\label{appendix:algo}
We provide pseudo-codes for \textit{VAST} as Algorithm~\ref{alg:train}.

\makeatletter
\newenvironment{breakablealgorithm}
{%
  \par\addvspace{0.45em}%
  \noindent\begin{minipage}{\linewidth}%
  \refstepcounter{algorithm}%
  \hrule height.8pt depth0pt \kern2pt
  \noindent{\bfseries\ALG@name~\thealgorithm} \itshape
}{%
  \par\kern2pt\hrule height.8pt depth0pt
  \end{minipage}%
  \par\addvspace{0.45em}%
}
\makeatother
\begin{breakablealgorithm}
\textit{VAST}\label{alg:train}\par
\kern2pt\hrule height.4pt depth0pt \kern2pt
\normalfont
\begin{algorithmic}
    \STATE \textcolor{byebyebabeblue}{\textbf{// Train time:}}
    \STATE Initialize $G$, $V$, $Q$, $\pi_\text{stitch}$, $\pi_\text{exec}$
    \WHILE{not converged}
        \STATE Sample batch ${(s_t,a_t,r_t,s_{t+1},s_\text{goal},s_k,a_k,k,s_i,i) \sim D}$ from offline data
        \STATE \textcolor{byebyebabeblue}{// Policy Learning}
        \STATE Update stitching policy $\pi_\text{stitch}(s_k,k \mid s)$ with loss (Eq.~(\ref{eql:weighted_high}))
        \STATE Update execute policy $\pi_\text{exec}(a \mid s, s_\text{goal})$ with loss (Eq.~(\ref{eql:bc_low}))
        \STATE \textcolor{byebyebabeblue}{// Value Learning}
        \STATE Update $G$ with loss (Eq.~(\ref{eql:g_loss}))
        \STATE Update $V$ with loss (Eq.~(\ref{eql:v_loss}))
        \STATE Update $Q$ with loss (Eq.~(\ref{eql:q_loss}))
    \ENDWHILE
    \STATE 
    \STATE \textcolor{byebyebabeblue}{\textbf{// Test time:}}
    \STATE Get initial state $s$
    \WHILE{not done}
        \STATE \textcolor{byebyebabeblue}{// Sub-goal prediction}
        \STATE Sample $s_\text{goal}$ from $\pi_\text{stitch}(\cdot | s)$
        \STATE \textcolor{byebyebabeblue}{// Execution}
        \STATE Sample $a_1,a_2,...,a_N$ from $\pi_\text{exec}(\cdot | s,s_\text{goal})$
        \STATE Set $a \leftarrow \arg \max_{a_1,a_2,...,a_N} Q(s,a_i)$
        \STATE Update $s$
    \ENDWHILE
\end{algorithmic}
\end{breakablealgorithm}

\section{Additional Results}
\label{appendix:fullresults}

\subsection{Full Results}
\label{appendix:fullresults_table}
In this section, we provide the full experimental results on OGBench (Table~\ref{table:ogbench_full}). Scores are aggregated over 8 random seeds, with $\pm$ indicating standard deviations. Best performances are highlighted with color. See Section~\ref{appendix:experimental_evaluation} for the evaluation portal.

\begin{table}[H]
\centering
\caption{\label{table:ogbench_full} \textbf{Results on OGBench.}}
\resizebox{0.95\linewidth}{!}{\scriptsize
\begin{tabular}{lcccccccccc}
\toprule
&
& \multicolumn{3}{>{\columncolor[HTML]{EEEEEE}}c}{\textbf{1-step TD}}
& \multicolumn{2}{>{\columncolor[HTML]{EEEEEE}}c}{\textbf{n-step TD}}
& \multicolumn{1}{>{\columncolor[HTML]{EEEEEE}}c}{\textbf{Chunking}}
& \multicolumn{2}{>{\columncolor[HTML]{EEEEEE}}c}{\textbf{Generative}}
& \multicolumn{1}{>{\columncolor[HTML]{EEEEEE}}c}{\textbf{Stitching}} \\
\cmidrule(lr){3-5} \cmidrule(lr){6-7} \cmidrule(lr){8-8} \cmidrule(lr){9-10} \cmidrule(lr){11-11}
& \textbf{Task} & IQL & FQL & IFQL & FQL-n & IFQL-n & QC & VALUE-FLOWS & FLOQ & VAST \\
\midrule
\multirow{6}{*}{\texttt{scene-play}}  & 1     
                        & $\textrm{94}$ 
                        & \colorbox{tableblue}{$\textrm{100} \textcolor{lightgray}{\pm \textrm{0}}$} 
                        & $\textrm{99} \textcolor{lightgray}{\pm \textrm{1}}$ 
                        & \colorbox{tableblue}{$\textrm{100} \textcolor{lightgray}{\pm \textrm{0}}$} 
                        & $\textrm{99} \textcolor{lightgray}{\pm \textrm{1}}$ 
                        & \colorbox{tableblue}{$\textrm{100} \textcolor{lightgray}{\pm \textrm{0}}$}
                        & \colorbox{tableblue}{$\textrm{100} \textcolor{lightgray}{\pm \textrm{1}}$} 
                        & \colorbox{tableblue}{$\textrm{100} \textcolor{lightgray}{\pm \textrm{0}}$} 
                        & \colorbox{tableblue}{$\textrm{100} \textcolor{lightgray}{\pm \textrm{0}}$} 
\\
                        & 2     
                        & $\textrm{12}$ 
                        & $\textrm{81} \textcolor{lightgray}{\pm \textrm{14}}$ 
                        & $\textrm{55} \textcolor{lightgray}{\pm \textrm{26}}$ 
                        & $\textrm{26} \textcolor{lightgray}{\pm \textrm{15}}$ 
                        & $\textrm{0} \textcolor{lightgray}{\pm \textrm{0}}$ 
                        & \colorbox{tableblue}{$\textrm{99} \textcolor{lightgray}{\pm \textrm{2}}$} 
                        & \colorbox{tableblue}{$\textrm{99} \textcolor{lightgray}{\pm \textrm{1}}$} 
                        & $\textrm{82} \textcolor{lightgray}{\pm \textrm{11}}$ 
                        & $\textrm{98} \textcolor{lightgray}{\pm \textrm{3}}$
\\
                        & 3     
                        & $\textrm{32}$ 
                        & $\textrm{97} \textcolor{lightgray}{\pm \textrm{3}}$ 
                        & $\textrm{76} \textcolor{lightgray}{\pm \textrm{9}}$ 
                        & $\textrm{85} \textcolor{lightgray}{\pm \textrm{11}}$ 
                        & $\textrm{59} \textcolor{lightgray}{\pm \textrm{4}}$ 
                        & \colorbox{tableblue}{$\textrm{99} \textcolor{lightgray}{\pm \textrm{1}}$} 
                        & $\textrm{95} \textcolor{lightgray}{\pm \textrm{6}}$ 
                        & $\textrm{98} \textcolor{lightgray}{\pm \textrm{2}}$ 
                        & $\textrm{97} \textcolor{lightgray}{\pm \textrm{3}}$
\\
                        & 4     
                        & $\textrm{0}$ 
                        & \colorbox{tableblue}{$\textrm{9} \textcolor{lightgray}{\pm \textrm{7}}$} 
                        & $\textrm{0} \textcolor{lightgray}{\pm \textrm{0}}$ 
                        & $\textrm{0} \textcolor{lightgray}{\pm \textrm{0}}$ 
                        & $\textrm{0} \textcolor{lightgray}{\pm \textrm{1}}$ 
                        & $\textrm{1} \textcolor{lightgray}{\pm \textrm{2}}$ 
                        & $\textrm{5} \textcolor{lightgray}{\pm \textrm{8}}$ 
                        & $\textrm{4} \textcolor{lightgray}{\pm \textrm{4}}$ 
                        & $\textrm{8} \textcolor{lightgray}{\pm \textrm{4}}$
\\
                        & 5     
                        & $\textrm{0}$ 
                        & $\textrm{0} \textcolor{lightgray}{\pm \textrm{0}}$ 
                        & $\textrm{0} \textcolor{lightgray}{\pm \textrm{0}}$ 
                        & $\textrm{0} \textcolor{lightgray}{\pm \textrm{0}}$ 
                        & $\textrm{0} \textcolor{lightgray}{\pm \textrm{0}}$ 
                        & $\textrm{0} \textcolor{lightgray}{\pm \textrm{0}}$ 
                        & $\textrm{0} \textcolor{lightgray}{\pm \textrm{0}}$ 
                        & $\textrm{0} \textcolor{lightgray}{\pm \textrm{0}}$ 
                        & \colorbox{tableblue}{$\textrm{1} \textcolor{lightgray}{\pm \textrm{2}}$}
\\
\cmidrule(lr){2-11}
                        & agg.     
                        & $\textrm{28}$ 
                        & $\textrm{57} \textcolor{lightgray}{\pm \textrm{7}}$ 
                        & $\textrm{46} \textcolor{lightgray}{\pm \textrm{12}}$ 
                        & $\textrm{42} \textcolor{lightgray}{\pm \textrm{8}}$ 
                        & $\textrm{32} \textcolor{lightgray}{\pm \textrm{2}}$ 
                        & $\textrm{60} \textcolor{lightgray}{\pm \textrm{1}}$
                        & $\textrm{60} \textcolor{lightgray}{\pm \textrm{5}}$
                        & $\textrm{57} \textcolor{lightgray}{\pm \textrm{5}}$ 
                        & \colorbox{tableblue}{$\textrm{61} \textcolor{lightgray}{\pm \textrm{3}}$}
\\
\cmidrule(lr){1-11}
\multirow{6}{*}{\texttt{cube-single-play}}    & 1 
                        & $\textrm{88}$ 
                        & $\textrm{96} \textcolor{lightgray}{\pm \textrm{7}}$ 
                        & $\textrm{84} \textcolor{lightgray}{\pm \textrm{7}}$ 
                        & $\textrm{97} \textcolor{lightgray}{\pm \textrm{2}}$ 
                        & $\textrm{77} \textcolor{lightgray}{\pm \textrm{8}}$ 
                        & $\textrm{98} \textcolor{lightgray}{\pm \textrm{2}}$ 
                        & $\textrm{97} \textcolor{lightgray}{\pm \textrm{2}}$ 
                        & \colorbox{tableblue}{$\textrm{99} \textcolor{lightgray}{\pm \textrm{1}}$}
                        & \colorbox{tableblue}{$\textrm{99} \textcolor{lightgray}{\pm \textrm{2}}$}
\\
                        & 2     
                        & $\textrm{85}$ 
                        & $\textrm{97} \textcolor{lightgray}{\pm \textrm{3}}$ 
                        & $\textrm{89} \textcolor{lightgray}{\pm \textrm{5}}$ 
                        & $\textrm{96} \textcolor{lightgray}{\pm \textrm{1}}$ 
                        & $\textrm{86} \textcolor{lightgray}{\pm \textrm{5}}$ 
                        & $\textrm{94} \textcolor{lightgray}{\pm \textrm{4}}$ 
                        & $\textrm{97} \textcolor{lightgray}{\pm \textrm{4}}$ 
                        & $\textrm{98} \textcolor{lightgray}{\pm \textrm{3}}$
                        & \colorbox{tableblue}{$\textrm{100} \textcolor{lightgray}{\pm \textrm{1}}$}
\\
                        & 3     
                        & $\textrm{91}$ 
                        & $\textrm{98} \textcolor{lightgray}{\pm \textrm{3}}$ 
                        & $\textrm{92} \textcolor{lightgray}{\pm \textrm{3}}$ 
                        & $\textrm{99} \textcolor{lightgray}{\pm \textrm{1}}$ 
                        & $\textrm{87} \textcolor{lightgray}{\pm \textrm{6}}$ 
                        & $\textrm{97} \textcolor{lightgray}{\pm \textrm{4}}$ 
                        & $\textrm{99} \textcolor{lightgray}{\pm \textrm{1}}$ 
                        & $\textrm{99} \textcolor{lightgray}{\pm \textrm{2}}$ 
                        & \colorbox{tableblue}{$\textrm{100} \textcolor{lightgray}{\pm \textrm{1}}$}
\\
                        & 4     
                        & $\textrm{73}$ 
                        & $\textrm{96} \textcolor{lightgray}{\pm \textrm{5}}$ 
                        & $\textrm{87} \textcolor{lightgray}{\pm \textrm{4}}$ 
                        & \colorbox{tableblue}{$\textrm{98} \textcolor{lightgray}{\pm \textrm{2}}$} 
                        & $\textrm{78} \textcolor{lightgray}{\pm \textrm{6}}$ 
                        & $\textrm{89} \textcolor{lightgray}{\pm \textrm{4}}$ 
                        & $\textrm{95} \textcolor{lightgray}{\pm \textrm{3}}$ 
                        & $\textrm{96} \textcolor{lightgray}{\pm \textrm{3}}$ 
                        & \colorbox{tableblue}{$\textrm{98} \textcolor{lightgray}{\pm \textrm{3}}$}
\\
                        & 5     
                        & $\textrm{78}$ 
                        & $\textrm{96} \textcolor{lightgray}{\pm \textrm{3}}$ 
                        & $\textrm{84} \textcolor{lightgray}{\pm \textrm{7}}$ 
                        & \colorbox{tableblue}{$\textrm{97} \textcolor{lightgray}{\pm \textrm{5}}$} 
                        & $\textrm{84} \textcolor{lightgray}{\pm \textrm{5}}$ 
                        & $\textrm{89} \textcolor{lightgray}{\pm \textrm{9}}$ 
                        & $\textrm{86} \textcolor{lightgray}{\pm \textrm{8}}$ 
                        & $\textrm{89} \textcolor{lightgray}{\pm \textrm{9}}$ 
                        & $\textrm{94} \textcolor{lightgray}{\pm \textrm{4}}$
\\
\cmidrule(lr){2-11}
                        & agg.     
                        & $\textrm{83}$ 
                        & $\textrm{97} \textcolor{lightgray}{\pm \textrm{4}}$ 
                        & $\textrm{87} \textcolor{lightgray}{\pm \textrm{5}}$ 
                        & $\textrm{97} \textcolor{lightgray}{\pm \textrm{3}}$
                        & $\textrm{82} \textcolor{lightgray}{\pm \textrm{6}}$ 
                        & $\textrm{93} \textcolor{lightgray}{\pm \textrm{5}}$ 
                        & $\textrm{95} \textcolor{lightgray}{\pm \textrm{4}}$ 
                        & $\textrm{96} \textcolor{lightgray}{\pm \textrm{5}}$ 
                        & \colorbox{tableblue}{$\textrm{98} \textcolor{lightgray}{\pm \textrm{2}}$}
\\
\cmidrule(lr){1-11}
\multirow{6}{*}{\texttt{cube-double-play}}    & 1 
                        & $\textrm{27}$ 
                        & $\textrm{67} \textcolor{lightgray}{\pm \textrm{9}}$ 
                        & $\textrm{32} \textcolor{lightgray}{\pm \textrm{6}}$ 
                        & $\textrm{16} \textcolor{lightgray}{\pm \textrm{6}}$ 
                        & $\textrm{9} \textcolor{lightgray}{\pm \textrm{4}}$ 
                        & $\textrm{85} \textcolor{lightgray}{\pm \textrm{6}}$ 
                        & \colorbox{tableblue}{$\textrm{96} \textcolor{lightgray}{\pm \textrm{5}}$} 
                        & $\textrm{64} \textcolor{lightgray}{\pm \textrm{17}}$ 
                        & $\textrm{93} \textcolor{lightgray}{\pm \textrm{2}}$
\\
                        & 2     
                        & $\textrm{1}$ 
                        & $\textrm{38} \textcolor{lightgray}{\pm \textrm{10}}$ 
                        & $\textrm{11} \textcolor{lightgray}{\pm \textrm{4}}$ 
                        & $\textrm{1} \textcolor{lightgray}{\pm \textrm{1}}$ 
                        & $\textrm{7} \textcolor{lightgray}{\pm \textrm{5}}$ 
                        & $\textrm{78} \textcolor{lightgray}{\pm \textrm{6}}$ 
                        & $\textrm{66} \textcolor{lightgray}{\pm \textrm{11}}$ 
                        & $\textrm{61} \textcolor{lightgray}{\pm \textrm{22}}$ 
                        & \colorbox{tableblue}{$\textrm{87} \textcolor{lightgray}{\pm \textrm{4}}$}
\\
                        & 3     
                        & $\textrm{0}$ 
                        & $\textrm{23} \textcolor{lightgray}{\pm \textrm{6}}$ 
                        & $\textrm{4} \textcolor{lightgray}{\pm \textrm{3}}$ 
                        & $\textrm{1} \textcolor{lightgray}{\pm \textrm{1}}$ 
                        & $\textrm{4} \textcolor{lightgray}{\pm \textrm{3}}$ 
                        & $\textrm{74} \textcolor{lightgray}{\pm \textrm{5}}$ 
                        & $\textrm{63} \textcolor{lightgray}{\pm \textrm{15}}$ 
                        & $\textrm{57} \textcolor{lightgray}{\pm \textrm{12}}$ 
                        & \colorbox{tableblue}{$\textrm{85} \textcolor{lightgray}{\pm \textrm{9}}$}
\\
                        & 4     
                        & $\textrm{0}$ 
                        & $\textrm{5} \textcolor{lightgray}{\pm \textrm{2}}$ 
                        & $\textrm{1} \textcolor{lightgray}{\pm \textrm{2}}$ 
                        & $\textrm{0} \textcolor{lightgray}{\pm \textrm{0}}$ 
                        & $\textrm{0} \textcolor{lightgray}{\pm \textrm{0}}$ 
                        & \colorbox{tableblue}{$\textrm{26} \textcolor{lightgray}{\pm \textrm{9}}$} 
                        & $\textrm{20} \textcolor{lightgray}{\pm \textrm{11}}$ 
                        & $\textrm{12} \textcolor{lightgray}{\pm \textrm{9}}$ 
                        & $\textrm{22} \textcolor{lightgray}{\pm \textrm{7}}$
\\
                        & 5     
                        & $\textrm{4}$ 
                        & $\textrm{17} \textcolor{lightgray}{\pm \textrm{9}}$ 
                        & $\textrm{10} \textcolor{lightgray}{\pm \textrm{6}}$ 
                        & $\textrm{2} \textcolor{lightgray}{\pm \textrm{1}}$ 
                        & $\textrm{2} \textcolor{lightgray}{\pm \textrm{2}}$ 
                        & $\textrm{72} \textcolor{lightgray}{\pm \textrm{9}}$
                        & $\textrm{61} \textcolor{lightgray}{\pm \textrm{9}}$ 
                        & $\textrm{52} \textcolor{lightgray}{\pm \textrm{15}}$ 
                        & \colorbox{tableblue}{$\textrm{89} \textcolor{lightgray}{\pm \textrm{7}}$}
\\
\cmidrule(lr){2-11}
                        & agg.     
                        & $\textrm{7}$ 
                        & $\textrm{30} \textcolor{lightgray}{\pm \textrm{8}}$ 
                        & $\textrm{12} \textcolor{lightgray}{\pm \textrm{5}}$ 
                        & $\textrm{4} \textcolor{lightgray}{\pm \textrm{3}}$ 
                        & $\textrm{4} \textcolor{lightgray}{\pm \textrm{3}}$ 
                        & $\textrm{67} \textcolor{lightgray}{\pm \textrm{7}}$
                        & $\textrm{61} \textcolor{lightgray}{\pm \textrm{11}}$ 
                        & $\textrm{49} \textcolor{lightgray}{\pm \textrm{16}}$ 
                        & \colorbox{tableblue}{$\textrm{75} \textcolor{lightgray}{\pm \textrm{6}}$}
\\
\cmidrule(lr){1-11}
\multirow{6}{*}{\texttt{puzzle-4x4-play}}    & 1 
                        & $\textrm{12}$ 
                        & $\textrm{29} \textcolor{lightgray}{\pm \textrm{6}}$ 
                        & $\textrm{41} \textcolor{lightgray}{\pm \textrm{12}}$ 
                        & $\textrm{42} \textcolor{lightgray}{\pm \textrm{4}}$ 
                        & $\textrm{17} \textcolor{lightgray}{\pm \textrm{6}}$ 
                        & $\textrm{67} \textcolor{lightgray}{\pm \textrm{11}}$
                        & $\textrm{18} \textcolor{lightgray}{\pm \textrm{12}}$ 
                        & $\textrm{61} \textcolor{lightgray}{\pm \textrm{8}}$ 
                        & \colorbox{tableblue}{$\textrm{69} \textcolor{lightgray}{\pm \textrm{11}}$}
\\
                        & 2     
                        & $\textrm{7}$ 
                        & $\textrm{14} \textcolor{lightgray}{\pm \textrm{7}}$ 
                        & $\textrm{17} \textcolor{lightgray}{\pm \textrm{9}}$ 
                        & $\textrm{20} \textcolor{lightgray}{\pm \textrm{3}}$ 
                        & $\textrm{9} \textcolor{lightgray}{\pm \textrm{4}}$ 
                        & $\textrm{1} \textcolor{lightgray}{\pm \textrm{2}}$ 
                        & \colorbox{tableblue}{$\textrm{33} \textcolor{lightgray}{\pm \textrm{2}}$} 
                        & $\textrm{22} \textcolor{lightgray}{\pm \textrm{8}}$ 
                        & $\textrm{27} \textcolor{lightgray}{\pm \textrm{10}}$
\\
                        & 3     
                        & $\textrm{9}$ 
                        & $\textrm{18} \textcolor{lightgray}{\pm \textrm{6}}$ 
                        & $\textrm{48} \textcolor{lightgray}{\pm \textrm{8}}$ 
                        & $\textrm{35} \textcolor{lightgray}{\pm \textrm{8}}$ 
                        & $\textrm{15} \textcolor{lightgray}{\pm \textrm{3}}$ 
                        & $\textrm{64} \textcolor{lightgray}{\pm \textrm{17}}$
                        & $\textrm{18} \textcolor{lightgray}{\pm \textrm{12}}$ 
                        & $\textrm{40} \textcolor{lightgray}{\pm \textrm{15}}$ 
                        & \colorbox{tableblue}{$\textrm{81} \textcolor{lightgray}{\pm \textrm{11}}$}
\\
                        & 4     
                        & $\textrm{5}$ 
                        & $\textrm{9} \textcolor{lightgray}{\pm \textrm{2}}$ 
                        & $\textrm{20} \textcolor{lightgray}{\pm \textrm{6}}$ 
                        & $\textrm{20} \textcolor{lightgray}{\pm \textrm{5}}$ 
                        & $\textrm{10} \textcolor{lightgray}{\pm \textrm{3}}$ 
                        & $\textrm{10} \textcolor{lightgray}{\pm \textrm{8}}$ 
                        & $\textrm{20} \textcolor{lightgray}{\pm \textrm{5}}$ 
                        & $\textrm{19} \textcolor{lightgray}{\pm \textrm{6}}$ 
                        & \colorbox{tableblue}{$\textrm{33} \textcolor{lightgray}{\pm \textrm{10}}$}
\\
                        & 5     
                        & $\textrm{4}$ 
                        & $\textrm{6} \textcolor{lightgray}{\pm \textrm{5}}$ 
                        & $\textrm{7} \textcolor{lightgray}{\pm \textrm{3}}$ 
                        & $\textrm{17} \textcolor{lightgray}{\pm \textrm{7}}$
                        & $\textrm{7} \textcolor{lightgray}{\pm \textrm{4}}$ 
                        & $\textrm{3} \textcolor{lightgray}{\pm \textrm{3}}$ 
                        & $\textrm{14} \textcolor{lightgray}{\pm \textrm{6}}$ 
                        & $\textrm{14} \textcolor{lightgray}{\pm \textrm{4}}$ 
                        & \colorbox{tableblue}{$\textrm{21} \textcolor{lightgray}{\pm \textrm{5}}$}
\\
\cmidrule(lr){2-11}
                        & agg.     
                        & $\textrm{7}$ 
                        & $\textrm{15} \textcolor{lightgray}{\pm \textrm{5}}$ 
                        & $\textrm{27} \textcolor{lightgray}{\pm \textrm{8}}$ 
                        & $\textrm{27} \textcolor{lightgray}{\pm \textrm{6}}$ 
                        & $\textrm{12} \textcolor{lightgray}{\pm \textrm{4}}$ 
                        & $\textrm{29} \textcolor{lightgray}{\pm \textrm{10}}$ 
                        & $\textrm{21} \textcolor{lightgray}{\pm \textrm{8}}$ 
                        & $\textrm{31} \textcolor{lightgray}{\pm \textrm{9}}$ 
                        & \colorbox{tableblue}{$\textrm{46} \textcolor{lightgray}{\pm \textrm{10}}$}
\\
\cmidrule(lr){1-11}
\multirow{6}{*}{\texttt{antmaze-large-navigate}}    & 1 
                        & $\textrm{48}$ 
                        & $\textrm{87} \textcolor{lightgray}{\pm \textrm{6}}$ 
                        & $\textrm{34} \textcolor{lightgray}{\pm \textrm{20}}$ 
                        & $\textrm{80} \textcolor{lightgray}{\pm \textrm{10}}$ 
                        & $\textrm{15} \textcolor{lightgray}{\pm \textrm{10}}$ 
                        & $\textrm{4} \textcolor{lightgray}{\pm \textrm{5}}$ 
                        & $\textrm{22} \textcolor{lightgray}{\pm \textrm{11}}$ 
                        & \colorbox{tableblue}{$\textrm{96} \textcolor{lightgray}{\pm \textrm{3}}$} 
                        & $\textrm{87} \textcolor{lightgray}{\pm \textrm{7}}$
\\
                        & 2     
                        & $\textrm{42}$ 
                        & $\textrm{63} \textcolor{lightgray}{\pm \textrm{10}}$ 
                        & $\textrm{17} \textcolor{lightgray}{\pm \textrm{7}}$ 
                        & $\textrm{59} \textcolor{lightgray}{\pm \textrm{14}}$ 
                        & $\textrm{0} \textcolor{lightgray}{\pm \textrm{0}}$ 
                        & $\textrm{0} \textcolor{lightgray}{\pm \textrm{0}}$ 
                        & $\textrm{53} \textcolor{lightgray}{\pm \textrm{14}}$ 
                        &\colorbox{tableblue}{$\textrm{83} \textcolor{lightgray}{\pm \textrm{11}}$} 
                        & $\textrm{74} \textcolor{lightgray}{\pm \textrm{2}}$
\\
                        & 3     
                        & $\textrm{72}$ 
                        & $\textrm{96} \textcolor{lightgray}{\pm \textrm{3}}$ 
                        & $\textrm{49} \textcolor{lightgray}{\pm \textrm{12}}$ 
                        & $\textrm{91} \textcolor{lightgray}{\pm \textrm{5}}$ 
                        & $\textrm{57} \textcolor{lightgray}{\pm \textrm{8}}$ 
                        & $\textrm{19} \textcolor{lightgray}{\pm \textrm{10}}$ 
                        & $\textrm{53} \textcolor{lightgray}{\pm \textrm{23}}$ 
                        & \colorbox{tableblue}{$\textrm{98} \textcolor{lightgray}{\pm \textrm{2}}$} 
                        & $\textrm{81} \textcolor{lightgray}{\pm \textrm{8}}$
\\
                        & 4     
                        & $\textrm{51}$ 
                        & $\textrm{80} \textcolor{lightgray}{\pm \textrm{8}}$ 
                        & $\textrm{15} \textcolor{lightgray}{\pm \textrm{13}}$ 
                        & $\textrm{83} \textcolor{lightgray}{\pm \textrm{7}}$ 
                        & $\textrm{8} \textcolor{lightgray}{\pm \textrm{6}}$ 
                        & $\textrm{0} \textcolor{lightgray}{\pm \textrm{0}}$ 
                        & $\textrm{50} \textcolor{lightgray}{\pm \textrm{17}}$ 
                        & \colorbox{tableblue}{$\textrm{92} \textcolor{lightgray}{\pm \textrm{7}}$} 
                        & $\textrm{80} \textcolor{lightgray}{\pm \textrm{10}}$
\\
                        & 5     
                        & $\textrm{54}$ 
                        & $\textrm{85} \textcolor{lightgray}{\pm \textrm{6}}$ 
                        & $\textrm{44} \textcolor{lightgray}{\pm \textrm{26}}$ 
                        & $\textrm{86} \textcolor{lightgray}{\pm \textrm{6}}$ 
                        & $\textrm{56} \textcolor{lightgray}{\pm \textrm{7}}$ 
                        & $\textrm{8} \textcolor{lightgray}{\pm \textrm{3}}$ 
                        & $\textrm{61} \textcolor{lightgray}{\pm \textrm{5}}$ 
                        & \colorbox{tableblue}{$\textrm{95} \textcolor{lightgray}{\pm \textrm{3}}$} 
                        & $\textrm{84} \textcolor{lightgray}{\pm \textrm{2}}$
\\
\cmidrule(lr){2-11}
                        & agg.     
                        & $\textrm{53}$ 
                        & $\textrm{82} \textcolor{lightgray}{\pm \textrm{7}}$ 
                        & $\textrm{32} \textcolor{lightgray}{\pm \textrm{17}}$ 
                        & $\textrm{80} \textcolor{lightgray}{\pm \textrm{9}}$ 
                        & $\textrm{27} \textcolor{lightgray}{\pm \textrm{7}}$ 
                        & $\textrm{6} \textcolor{lightgray}{\pm \textrm{5}}$ 
                        & $\textrm{48} \textcolor{lightgray}{\pm \textrm{15}}$ 
                        & \colorbox{tableblue}{$\textrm{93} \textcolor{lightgray}{\pm \textrm{6}}$} 
                        & $\textrm{81} \textcolor{lightgray}{\pm \textrm{7}}$
\\
\cmidrule(lr){1-11}
\multirow{6}{*}{\texttt{antmaze-giant-navigate}}    & 1 
                        & $\textrm{0}$ 
                        & $\textrm{12} \textcolor{lightgray}{\pm \textrm{12}}$ 
                        & $\textrm{0} \textcolor{lightgray}{\pm \textrm{0}}$ 
                        & $\textrm{3} \textcolor{lightgray}{\pm \textrm{6}}$ 
                        & $\textrm{0} \textcolor{lightgray}{\pm \textrm{0}}$ 
                        & $\textrm{0} \textcolor{lightgray}{\pm \textrm{0}}$ 
                        & $\textrm{0} \textcolor{lightgray}{\pm \textrm{0}}$ 
                        & \colorbox{tableblue}{$\textrm{75} \textcolor{lightgray}{\pm \textrm{12}}$} 
                        & $\textrm{1} \textcolor{lightgray}{\pm \textrm{1}}$
\\
                        & 2     
                        & $\textrm{1}$ 
                        & $\textrm{21} \textcolor{lightgray}{\pm \textrm{18}}$ 
                        & $\textrm{0} \textcolor{lightgray}{\pm \textrm{0}}$ 
                        & \colorbox{tableblue}{$\textrm{88} \textcolor{lightgray}{\pm \textrm{3}}$} 
                        & $\textrm{0} \textcolor{lightgray}{\pm \textrm{0}}$ 
                        & $\textrm{0} \textcolor{lightgray}{\pm \textrm{0}}$ 
                        & $\textrm{1} \textcolor{lightgray}{\pm \textrm{1}}$ 
                        & $\textrm{37} \textcolor{lightgray}{\pm \textrm{30}}$ 
                        & $\textrm{46} \textcolor{lightgray}{\pm \textrm{7}}$
\\
                        & 3     
                        & $\textrm{0}$ 
                        & $\textrm{3} \textcolor{lightgray}{\pm \textrm{1}}$ 
                        & $\textrm{0} \textcolor{lightgray}{\pm \textrm{0}}$ 
                        & \colorbox{tableblue}{$\textrm{17} \textcolor{lightgray}{\pm \textrm{7}}$} 
                        & $\textrm{0} \textcolor{lightgray}{\pm \textrm{0}}$ 
                        & $\textrm{0} \textcolor{lightgray}{\pm \textrm{0}}$ 
                        & $\textrm{0} \textcolor{lightgray}{\pm \textrm{0}}$ 
                        & $\textrm{0} \textcolor{lightgray}{\pm \textrm{0}}$ 
                        & $\textrm{14} \textcolor{lightgray}{\pm \textrm{6}}$
\\
                        & 4     
                        & $\textrm{0}$ 
                        & $\textrm{10} \textcolor{lightgray}{\pm \textrm{21}}$ 
                        & $\textrm{0} \textcolor{lightgray}{\pm \textrm{0}}$ 
                        & \colorbox{tableblue}{$\textrm{40} \textcolor{lightgray}{\pm \textrm{32}}$} 
                        & $\textrm{0} \textcolor{lightgray}{\pm \textrm{0}}$ 
                        & $\textrm{0} \textcolor{lightgray}{\pm \textrm{0}}$ 
                        & $\textrm{0} \textcolor{lightgray}{\pm \textrm{0}}$ 
                        & $\textrm{17} \textcolor{lightgray}{\pm \textrm{28}}$ 
                        & $\textrm{7} \textcolor{lightgray}{\pm \textrm{5}}$
\\
                        & 5     
                        & $\textrm{19}$ 
                        & $\textrm{7} \textcolor{lightgray}{\pm \textrm{18}}$ 
                        & $\textrm{1} \textcolor{lightgray}{\pm \textrm{4}}$ 
                        & \colorbox{tableblue}{$\textrm{36} \textcolor{lightgray}{\pm \textrm{13}}$} 
                        & $\textrm{1} \textcolor{lightgray}{\pm \textrm{1}}$ 
                        & $\textrm{0} \textcolor{lightgray}{\pm \textrm{0}}$ 
                        & $\textrm{13} \textcolor{lightgray}{\pm \textrm{7}}$ 
                        & $\textrm{15} \textcolor{lightgray}{\pm \textrm{31}}$ 
                        & $\textrm{33} \textcolor{lightgray}{\pm \textrm{7}}$
\\
\cmidrule(lr){2-11}
                        & agg.     
                        & $\textrm{4}$ 
                        & $\textrm{11} \textcolor{lightgray}{\pm \textrm{16}}$ 
                        & $\textrm{0} \textcolor{lightgray}{\pm \textrm{2}}$ 
                        & \colorbox{tableblue}{$\textrm{37} \textcolor{lightgray}{\pm \textrm{16}}$} 
                        & $\textrm{0} \textcolor{lightgray}{\pm \textrm{0}}$ 
                        & $\textrm{0} \textcolor{lightgray}{\pm \textrm{0}}$ 
                        & $\textrm{3} \textcolor{lightgray}{\pm \textrm{3}}$ 
                        & $\textrm{29} \textcolor{lightgray}{\pm \textrm{24}}$ 
                        & $\textrm{20} \textcolor{lightgray}{\pm \textrm{6}}$
\\
\cmidrule(lr){1-11}
\multirow{6}{*}{\texttt{cube-triple-play}}    & 1 
                        & $\textrm{-}$ 
                        & $\textrm{6} \textcolor{lightgray}{\pm \textrm{8}}$ 
                        & $\textrm{3} \textcolor{lightgray}{\pm \textrm{2}}$ 
                        & $\textrm{1} \textcolor{lightgray}{\pm \textrm{1}}$ 
                        & $\textrm{1} \textcolor{lightgray}{\pm \textrm{1}}$ 
                        & $\textrm{19} \textcolor{lightgray}{\pm \textrm{11}}$ 
                        & $\textrm{53} \textcolor{lightgray}{\pm \textrm{22}}$ 
                        & $\textrm{17} \textcolor{lightgray}{\pm \textrm{10}}$ 
                        & \colorbox{tableblue}{$\textrm{61} \textcolor{lightgray}{\pm \textrm{17}}$}
\\
                        & 2     
                        & $\textrm{-}$ 
                        & $\textrm{0} \textcolor{lightgray}{\pm \textrm{0}}$ 
                        & $\textrm{0} \textcolor{lightgray}{\pm \textrm{0}}$ 
                        & $\textrm{0} \textcolor{lightgray}{\pm \textrm{0}}$ 
                        & $\textrm{0} \textcolor{lightgray}{\pm \textrm{0}}$ 
                        & $\textrm{2} \textcolor{lightgray}{\pm \textrm{2}}$ 
                        & $\textrm{0} \textcolor{lightgray}{\pm \textrm{0}}$ 
                        & $\textrm{0} \textcolor{lightgray}{\pm \textrm{0}}$ 
                        & \colorbox{tableblue}{$\textrm{10} \textcolor{lightgray}{\pm \textrm{4}}$}
\\
                        & 3     
                        & $\textrm{-}$ 
                        & $\textrm{0} \textcolor{lightgray}{\pm \textrm{0}}$ 
                        & $\textrm{0} \textcolor{lightgray}{\pm \textrm{0}}$ 
                        & $\textrm{0} \textcolor{lightgray}{\pm \textrm{0}}$ 
                        & $\textrm{0} \textcolor{lightgray}{\pm \textrm{0}}$ 
                        & $\textrm{1} \textcolor{lightgray}{\pm \textrm{1}}$ 
                        & $\textrm{8} \textcolor{lightgray}{\pm \textrm{5}}$ 
                        & $\textrm{1} \textcolor{lightgray}{\pm \textrm{2}}$ 
                        & \colorbox{tableblue}{$\textrm{10} \textcolor{lightgray}{\pm \textrm{6}}$}
\\
                        & 4     
                        & $\textrm{-}$ 
                        & $\textrm{0} \textcolor{lightgray}{\pm \textrm{0}}$ 
                        & $\textrm{0} \textcolor{lightgray}{\pm \textrm{0}}$ 
                        & $\textrm{0} \textcolor{lightgray}{\pm \textrm{0}}$ 
                        & $\textrm{0} \textcolor{lightgray}{\pm \textrm{0}}$ 
                        & $\textrm{0} \textcolor{lightgray}{\pm \textrm{0}}$ 
                        & $\textrm{1} \textcolor{lightgray}{\pm \textrm{2}}$ 
                        & $\textrm{0} \textcolor{lightgray}{\pm \textrm{0}}$ 
                        & \colorbox{tableblue}{$\textrm{2} \textcolor{lightgray}{\pm \textrm{2}}$}
\\
                        & 5     
                        & $\textrm{-}$ 
                        & $\textrm{0} \textcolor{lightgray}{\pm \textrm{0}}$ 
                        & $\textrm{0} \textcolor{lightgray}{\pm \textrm{0}}$ 
                        & $\textrm{0} \textcolor{lightgray}{\pm \textrm{0}}$ 
                        & $\textrm{0} \textcolor{lightgray}{\pm \textrm{0}}$ 
                        & $\textrm{0} \textcolor{lightgray}{\pm \textrm{0}}$ 
                        & $\textrm{1} \textcolor{lightgray}{\pm \textrm{1}}$ 
                        & $\textrm{0} \textcolor{lightgray}{\pm \textrm{0}}$ 
                        & \colorbox{tableblue}{$\textrm{2} \textcolor{lightgray}{\pm \textrm{2}}$}
\\
\cmidrule(lr){2-11}
                        & agg.     
                        & $\textrm{-}$ 
                        & $\textrm{1} \textcolor{lightgray}{\pm \textrm{4}}$ 
                        & $\textrm{1} \textcolor{lightgray}{\pm \textrm{1}}$ 
                        & $\textrm{0} \textcolor{lightgray}{\pm \textrm{0}}$ 
                        & $\textrm{0} \textcolor{lightgray}{\pm \textrm{0}}$ 
                        & $\textrm{4} \textcolor{lightgray}{\pm \textrm{5}}$ 
                        & $\textrm{13} \textcolor{lightgray}{\pm \textrm{10}}$ 
                        & $\textrm{4} \textcolor{lightgray}{\pm \textrm{5}}$ 
                        & \colorbox{tableblue}{$\textrm{17} \textcolor{lightgray}{\pm \textrm{8}}$}
\\
\cmidrule(lr){1-11}
\multirow{6}{*}{\texttt{cube-quadruple-play}}    & 1 
                        & $\textrm{-}$ 
                        & $\textrm{0} \textcolor{lightgray}{\pm \textrm{0}}$ 
                        & $\textrm{0} \textcolor{lightgray}{\pm \textrm{0}}$ 
                        & $\textrm{0} \textcolor{lightgray}{\pm \textrm{0}}$ 
                        & $\textrm{0} \textcolor{lightgray}{\pm \textrm{0}}$ 
                        & $\textrm{1} \textcolor{lightgray}{\pm \textrm{1}}$ 
                        & $\textrm{0} \textcolor{lightgray}{\pm \textrm{0}}$ 
                        & $\textrm{0} \textcolor{lightgray}{\pm \textrm{0}}$ 
                        & \colorbox{tableblue}{$\textrm{5} \textcolor{lightgray}{\pm \textrm{3}}$}
\\
                        & 2     
                        & $\textrm{-}$ 
                        & $\textrm{0} \textcolor{lightgray}{\pm \textrm{0}}$ 
                        & $\textrm{0} \textcolor{lightgray}{\pm \textrm{0}}$ 
                        & $\textrm{0} \textcolor{lightgray}{\pm \textrm{0}}$ 
                        & $\textrm{0} \textcolor{lightgray}{\pm \textrm{0}}$ 
                        & $\textrm{0} \textcolor{lightgray}{\pm \textrm{0}}$ 
                        & $\textrm{0} \textcolor{lightgray}{\pm \textrm{0}}$ 
                        & $\textrm{0} \textcolor{lightgray}{\pm \textrm{0}}$ 
                        & $\textrm{0} \textcolor{lightgray}{\pm \textrm{0}}$ 
\\
                        & 3     
                        & $\textrm{-}$ 
                        & $\textrm{0} \textcolor{lightgray}{\pm \textrm{0}}$ 
                        & $\textrm{0} \textcolor{lightgray}{\pm \textrm{0}}$ 
                        & $\textrm{0} \textcolor{lightgray}{\pm \textrm{0}}$ 
                        & $\textrm{0} \textcolor{lightgray}{\pm \textrm{0}}$ 
                        & $\textrm{0} \textcolor{lightgray}{\pm \textrm{0}}$ 
                        & $\textrm{0} \textcolor{lightgray}{\pm \textrm{0}}$ 
                        & $\textrm{0} \textcolor{lightgray}{\pm \textrm{0}}$ 
                        & $\textrm{0} \textcolor{lightgray}{\pm \textrm{0}}$ 
\\
                        & 4     
                        & $\textrm{-}$ 
                        & $\textrm{0} \textcolor{lightgray}{\pm \textrm{0}}$ 
                        & $\textrm{0} \textcolor{lightgray}{\pm \textrm{0}}$ 
                        & $\textrm{0} \textcolor{lightgray}{\pm \textrm{0}}$ 
                        & $\textrm{0} \textcolor{lightgray}{\pm \textrm{0}}$ 
                        & $\textrm{0} \textcolor{lightgray}{\pm \textrm{0}}$ 
                        & $\textrm{0} \textcolor{lightgray}{\pm \textrm{0}}$ 
                        & $\textrm{0} \textcolor{lightgray}{\pm \textrm{0}}$ 
                        & $\textrm{0} \textcolor{lightgray}{\pm \textrm{0}}$ 
\\
                        & 5     
                        & $\textrm{-}$ 
                        & $\textrm{0} \textcolor{lightgray}{\pm \textrm{0}}$ 
                        & $\textrm{0} \textcolor{lightgray}{\pm \textrm{0}}$ 
                        & $\textrm{0} \textcolor{lightgray}{\pm \textrm{0}}$ 
                        & $\textrm{0} \textcolor{lightgray}{\pm \textrm{0}}$ 
                        & $\textrm{0} \textcolor{lightgray}{\pm \textrm{0}}$ 
                        & $\textrm{0} \textcolor{lightgray}{\pm \textrm{0}}$ 
                        & $\textrm{0} \textcolor{lightgray}{\pm \textrm{0}}$ 
                        & $\textrm{0} \textcolor{lightgray}{\pm \textrm{0}}$ 
\\
\cmidrule(lr){2-11}
                        & agg.     
                        & $\textrm{-}$ 
                        & $\textrm{0} \textcolor{lightgray}{\pm \textrm{0}}$ 
                        & $\textrm{0} \textcolor{lightgray}{\pm \textrm{0}}$ 
                        & $\textrm{0} \textcolor{lightgray}{\pm \textrm{0}}$ 
                        & $\textrm{0} \textcolor{lightgray}{\pm \textrm{0}}$ 
                        & $\textrm{0} \textcolor{lightgray}{\pm \textrm{0}}$ 
                        & $\textrm{0} \textcolor{lightgray}{\pm \textrm{0}}$ 
                        & $\textrm{0} \textcolor{lightgray}{\pm \textrm{0}}$ 
                        & \colorbox{tableblue}{$\textrm{1} \textcolor{lightgray}{\pm \textrm{1}}$}
\\
\cmidrule(lr){1-11}
\multirow{6}{*}{\texttt{puzzle-4x5-play}}    & 1 
                        & $\textrm{-}$ 
                        & $\textrm{75} \textcolor{lightgray}{\pm \textrm{7}}$ 
                        & $\textrm{78} \textcolor{lightgray}{\pm \textrm{15}}$ 
                        & $\textrm{92} \textcolor{lightgray}{\pm \textrm{3}}$ 
                        & $\textrm{69} \textcolor{lightgray}{\pm \textrm{9}}$ 
                        & \colorbox{tableblue}{$\textrm{100} \textcolor{lightgray}{\pm \textrm{0}}$} 
                        & $\textrm{59} \textcolor{lightgray}{\pm \textrm{9}}$ 
                        & $\textrm{67} \textcolor{lightgray}{\pm \textrm{7}}$ 
                        & \colorbox{tableblue}{$\textrm{100} \textcolor{lightgray}{\pm \textrm{1}}$}
\\
                        & 2     
                        & $\textrm{-}$ 
                        & $\textrm{0} \textcolor{lightgray}{\pm \textrm{0}}$ 
                        & $\textrm{0} \textcolor{lightgray}{\pm \textrm{0}}$ 
                        & $\textrm{0} \textcolor{lightgray}{\pm \textrm{0}}$ 
                        & $\textrm{0} \textcolor{lightgray}{\pm \textrm{0}}$ 
                        & $\textrm{0} \textcolor{lightgray}{\pm \textrm{0}}$ 
                        & $\textrm{0} \textcolor{lightgray}{\pm \textrm{0}}$ 
                        & $\textrm{0} \textcolor{lightgray}{\pm \textrm{0}}$ 
                        & $\textrm{0} \textcolor{lightgray}{\pm \textrm{0}}$
\\
                        & 3     
                        & $\textrm{-}$ 
                        & $\textrm{0} \textcolor{lightgray}{\pm \textrm{0}}$ 
                        & $\textrm{0} \textcolor{lightgray}{\pm \textrm{0}}$ 
                        & $\textrm{0} \textcolor{lightgray}{\pm \textrm{0}}$ 
                        & $\textrm{0} \textcolor{lightgray}{\pm \textrm{0}}$ 
                        & $\textrm{0} \textcolor{lightgray}{\pm \textrm{0}}$ 
                        & $\textrm{0} \textcolor{lightgray}{\pm \textrm{0}}$ 
                        & $\textrm{0} \textcolor{lightgray}{\pm \textrm{0}}$ 
                        & $\textrm{0} \textcolor{lightgray}{\pm \textrm{0}}$ 
\\
                        & 4     
                        & $\textrm{-}$ 
                        & $\textrm{0} \textcolor{lightgray}{\pm \textrm{0}}$ 
                        & $\textrm{0} \textcolor{lightgray}{\pm \textrm{0}}$ 
                        & $\textrm{0} \textcolor{lightgray}{\pm \textrm{0}}$ 
                        & $\textrm{0} \textcolor{lightgray}{\pm \textrm{0}}$ 
                        & $\textrm{0} \textcolor{lightgray}{\pm \textrm{0}}$ 
                        & $\textrm{0} \textcolor{lightgray}{\pm \textrm{0}}$ 
                        & $\textrm{0} \textcolor{lightgray}{\pm \textrm{0}}$ 
                        & $\textrm{0} \textcolor{lightgray}{\pm \textrm{0}}$ 
\\
                        & 5     
                        & $\textrm{-}$ 
                        & $\textrm{0} \textcolor{lightgray}{\pm \textrm{0}}$ 
                        & $\textrm{0} \textcolor{lightgray}{\pm \textrm{0}}$ 
                        & $\textrm{0} \textcolor{lightgray}{\pm \textrm{0}}$ 
                        & $\textrm{0} \textcolor{lightgray}{\pm \textrm{0}}$ 
                        & $\textrm{0} \textcolor{lightgray}{\pm \textrm{0}}$ 
                        & $\textrm{0} \textcolor{lightgray}{\pm \textrm{0}}$ 
                        & $\textrm{0} \textcolor{lightgray}{\pm \textrm{0}}$ 
                        & $\textrm{0} \textcolor{lightgray}{\pm \textrm{0}}$ 
\\
\cmidrule(lr){2-11}
                        & agg.     
                        & $\textrm{-}$ 
                        & $\textrm{15} \textcolor{lightgray}{\pm \textrm{3}}$ 
                        & $\textrm{16} \textcolor{lightgray}{\pm \textrm{7}}$ 
                        & $\textrm{18} \textcolor{lightgray}{\pm \textrm{1}}$ 
                        & $\textrm{14} \textcolor{lightgray}{\pm \textrm{4}}$ 
                        & \colorbox{tableblue}{$\textrm{20} \textcolor{lightgray}{\pm \textrm{0}}$}
                        & $\textrm{12} \textcolor{lightgray}{\pm \textrm{4}}$ 
                        & $\textrm{13} \textcolor{lightgray}{\pm \textrm{3}}$ 
                        & \colorbox{tableblue}{$\textrm{20} \textcolor{lightgray}{\pm \textrm{0}}$}
\\
\cmidrule(lr){1-11}
\multirow{6}{*}{\texttt{puzzle-4x6-play}}    & 1 
                        & $\textrm{-}$ 
                        & $\textrm{20} \textcolor{lightgray}{\pm \textrm{12}}$ 
                        & $\textrm{0} \textcolor{lightgray}{\pm \textrm{0}}$ 
                        & $\textrm{18} \textcolor{lightgray}{\pm \textrm{19}}$ 
                        & $\textrm{15} \textcolor{lightgray}{\pm \textrm{12}}$ 
                        & $\textrm{21} \textcolor{lightgray}{\pm \textrm{10}}$ 
                        & $\textrm{17} \textcolor{lightgray}{\pm \textrm{13}}$ 
                        & $\textrm{12} \textcolor{lightgray}{\pm \textrm{6}}$ 
                        & \colorbox{tableblue}{$\textrm{38} \textcolor{lightgray}{\pm \textrm{38}}$}
\\
                        & 2     
                        & $\textrm{-}$ 
                        & $\textrm{0} \textcolor{lightgray}{\pm \textrm{1}}$ 
                        & $\textrm{0} \textcolor{lightgray}{\pm \textrm{0}}$ 
                        & $\textrm{0} \textcolor{lightgray}{\pm \textrm{0}}$ 
                        & \colorbox{tableblue}{$\textrm{1} \textcolor{lightgray}{\pm \textrm{1}}$} 
                        & $\textrm{0} \textcolor{lightgray}{\pm \textrm{0}}$ 
                        & $\textrm{0} \textcolor{lightgray}{\pm \textrm{0}}$ 
                        & \colorbox{tableblue}{$\textrm{1} \textcolor{lightgray}{\pm \textrm{1}}$} 
                        & $\textrm{0} \textcolor{lightgray}{\pm \textrm{0}}$ 
\\
                        & 3     
                        & $\textrm{-}$ 
                        & $\textrm{0} \textcolor{lightgray}{\pm \textrm{0}}$ 
                        & $\textrm{0} \textcolor{lightgray}{\pm \textrm{0}}$ 
                        & $\textrm{0} \textcolor{lightgray}{\pm \textrm{0}}$ 
                        & $\textrm{0} \textcolor{lightgray}{\pm \textrm{0}}$ 
                        & $\textrm{0} \textcolor{lightgray}{\pm \textrm{0}}$ 
                        & $\textrm{0} \textcolor{lightgray}{\pm \textrm{0}}$ 
                        & $\textrm{0} \textcolor{lightgray}{\pm \textrm{0}}$ 
                        & $\textrm{0} \textcolor{lightgray}{\pm \textrm{0}}$ 
\\
                        & 4     
                        & $\textrm{-}$ 
                        & $\textrm{0} \textcolor{lightgray}{\pm \textrm{0}}$ 
                        & $\textrm{0} \textcolor{lightgray}{\pm \textrm{0}}$ 
                        & $\textrm{0} \textcolor{lightgray}{\pm \textrm{0}}$ 
                        & $\textrm{0} \textcolor{lightgray}{\pm \textrm{0}}$ 
                        & $\textrm{0} \textcolor{lightgray}{\pm \textrm{0}}$ 
                        & $\textrm{0} \textcolor{lightgray}{\pm \textrm{0}}$ 
                        & $\textrm{0} \textcolor{lightgray}{\pm \textrm{0}}$ 
                        & $\textrm{0} \textcolor{lightgray}{\pm \textrm{0}}$ 
\\
                        & 5     
                        & $\textrm{-}$ 
                        & $\textrm{0} \textcolor{lightgray}{\pm \textrm{0}}$ 
                        & $\textrm{0} \textcolor{lightgray}{\pm \textrm{0}}$ 
                        & $\textrm{0} \textcolor{lightgray}{\pm \textrm{0}}$ 
                        & $\textrm{0} \textcolor{lightgray}{\pm \textrm{0}}$ 
                        & $\textrm{0} \textcolor{lightgray}{\pm \textrm{0}}$ 
                        & $\textrm{0} \textcolor{lightgray}{\pm \textrm{0}}$ 
                        & $\textrm{0} \textcolor{lightgray}{\pm \textrm{0}}$ 
                        & $\textrm{0} \textcolor{lightgray}{\pm \textrm{0}}$ 
\\
\cmidrule(lr){2-11}
                        & agg.     
                        & $\textrm{-}$ 
                        & $\textrm{4} \textcolor{lightgray}{\pm \textrm{5}}$ 
                        & $\textrm{0} \textcolor{lightgray}{\pm \textrm{0}}$
                        & $\textrm{4} \textcolor{lightgray}{\pm \textrm{9}}$
                        & $\textrm{3} \textcolor{lightgray}{\pm \textrm{5}}$ 
                        & $\textrm{4} \textcolor{lightgray}{\pm \textrm{4}}$ 
                        & $\textrm{3} \textcolor{lightgray}{\pm \textrm{6}}$ 
                        & $\textrm{3} \textcolor{lightgray}{\pm \textrm{3}}$ 
                        & \colorbox{tableblue}{$\textrm{8} \textcolor{lightgray}{\pm \textrm{17}}$}
\\
\bottomrule
\end{tabular}}
\end{table}

\subsection{Learning Curves}
\label{appendix:learning_curves}
We provide the learning curves on all OGBench tasks. Shaded regions indicate the minimum and maximum across seeds.

\newcommand{\ogbenchcurve}[2]{%
  \par\noindent\includegraphics[width=\textwidth]{#1}\par
  \captionof{figure}{\texttt{#2}.}\par
}
\begingroup
\centering
\setlength{\parskip}{0pt}%
\setlength{\abovecaptionskip}{4pt}%
\setlength{\belowcaptionskip}{0pt}%
\vspace{4pt}
\noindent\includegraphics[width=0.7\textwidth]{sections/img/curves/legend.png}\par
\vspace{6pt}
\ogbenchcurve{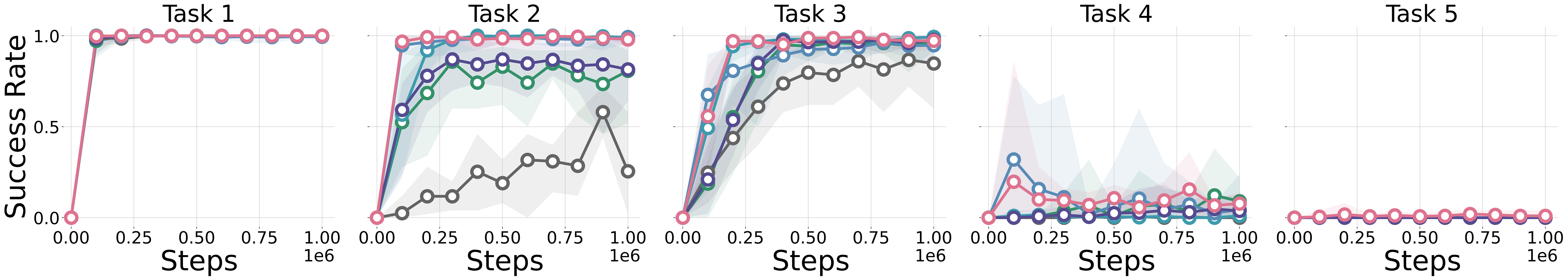}{scene-play}
\vspace{\fill}
\ogbenchcurve{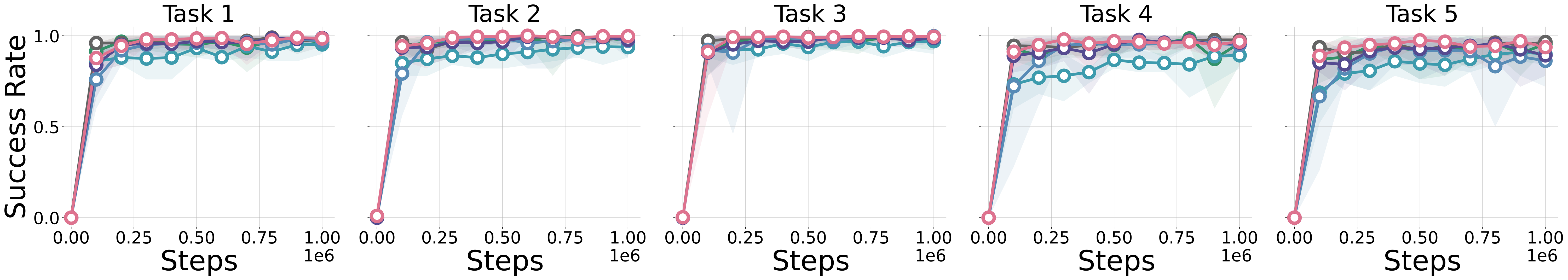}{cube-single-play}
\vspace{\fill}
\ogbenchcurve{sections/img/curves/cd-all-tasks_curve.png}{cube-double-play}
\vspace{\fill}
\ogbenchcurve{sections/img/curves/p44-all-tasks_curve.png}{puzzle-4x4-play}
\vspace{\fill}
\ogbenchcurve{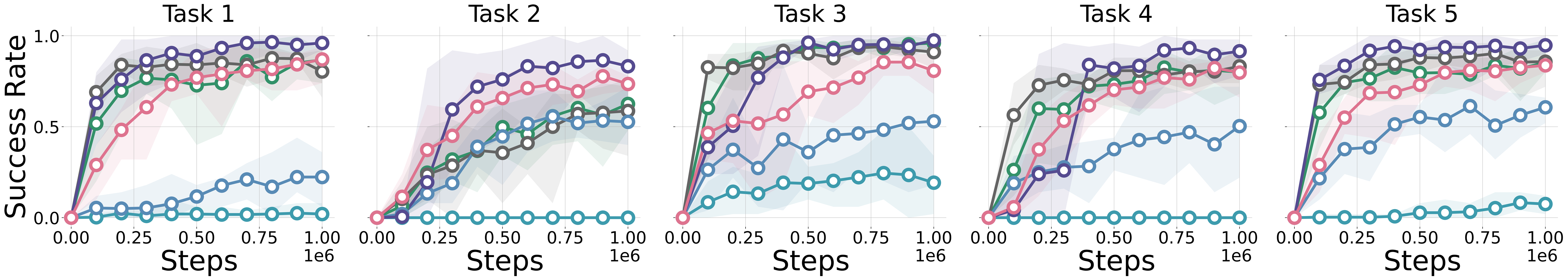}{antmaze-large-navigate}
\vspace{\fill}
\ogbenchcurve{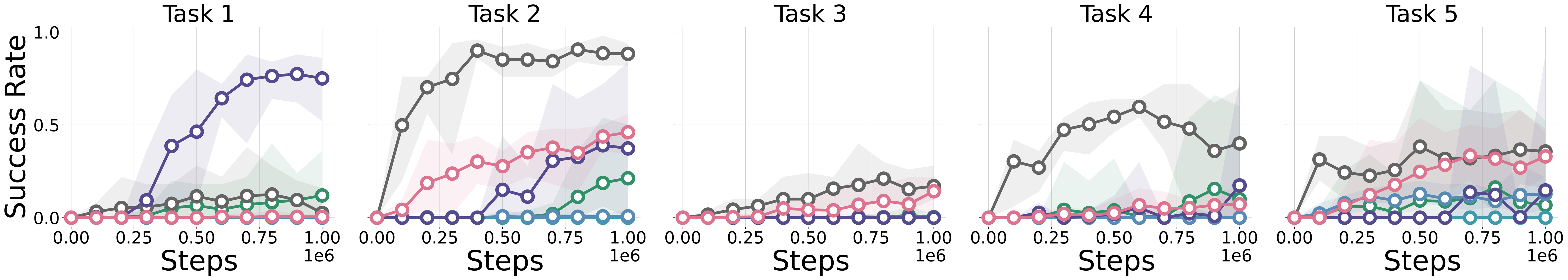}{antmaze-giant-navigate}
\endgroup
\clearpage
\begingroup
\centering
\setlength{\parskip}{0pt}%
\setlength{\abovecaptionskip}{4pt}%
\setlength{\belowcaptionskip}{0pt}%
\noindent\includegraphics[width=0.7\textwidth]{sections/img/curves/legend.png}\par
\vspace{6pt}
\ogbenchcurve{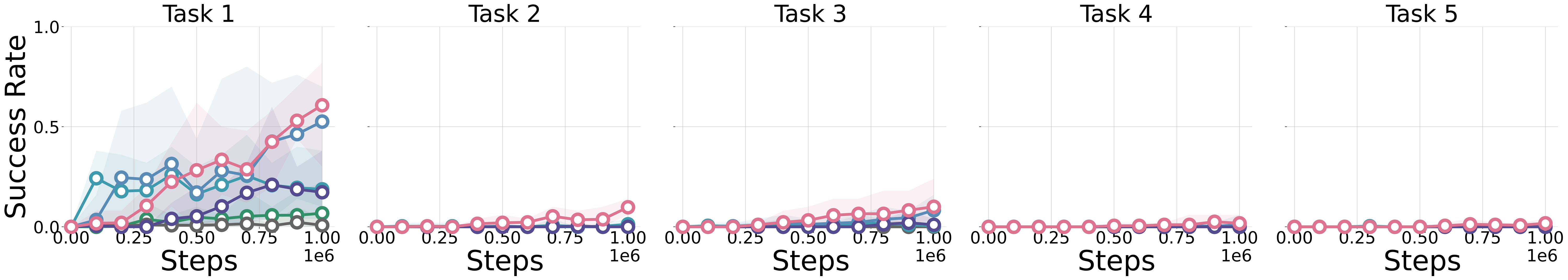}{cube-triple-play}
\vspace{18pt}
\ogbenchcurve{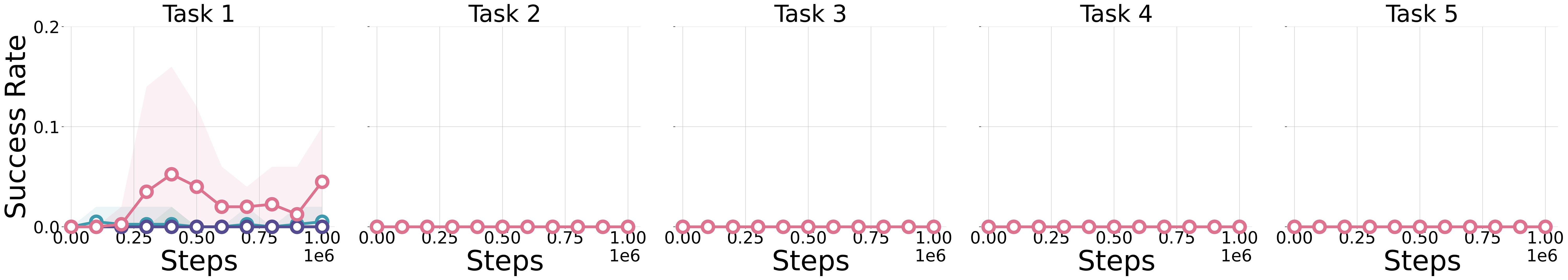}{cube-quadruple-play}
\vspace{18pt}
\ogbenchcurve{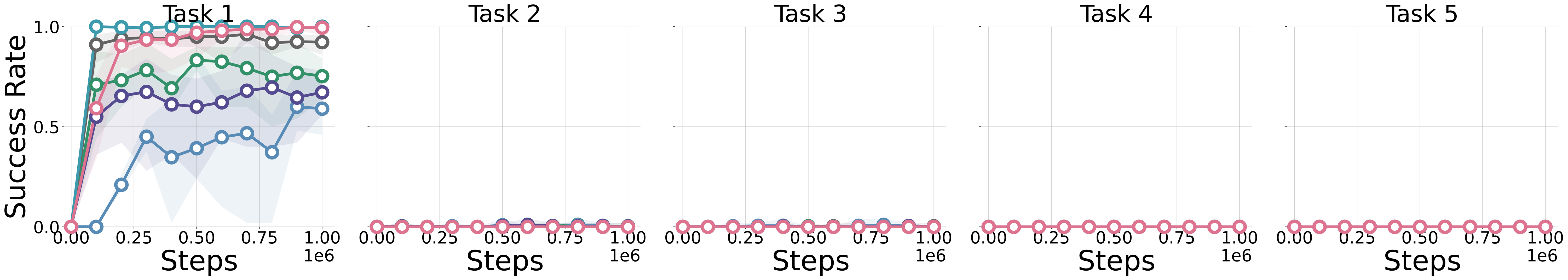}{puzzle-4x5-play}
\vspace{18pt}
\ogbenchcurve{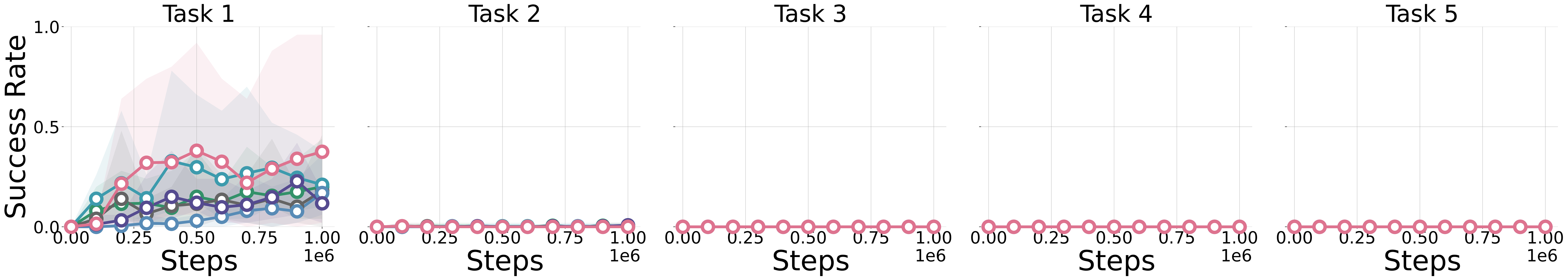}{puzzle-4x6-play}
\endgroup
\section{Additional Discussions}
\label{appendix:discussion}

This section provides discussions on important design choices: the relation of \textit{VAST} to model-based RL (Appendix~\ref{appendix:relation_to_model_based_rl}), why the $G$-function must be conditioned on the horizon (Appendix~\ref{appendix:why_horizon}), and the learning complexity of $G$ (Appendix~\ref{appendix:learning_complexity_of_g}).

\subsection{Comparison of \textit{VAST} to model-based RL}
\label{appendix:relation_to_model_based_rl}
Model-based RL shares the motivation of mitigating long-horizon value-estimation error by adopting temporally extended modeling.~\citet{park2026scalable} learns an action-chunk dynamics model that predicts future states from action sequences and uses imagined rollouts to assist value learning and policy improvement. While \textit{VAST} is fundamentally model-free that never learns a dynamics model nor generates imagined transitions. Instead, the $G$-function evaluates the return between \emph{actual} state pairs observed in the dataset, and $\pi_\text{stitch}$ optimizes directly over actionable, in-support sub-goals. The predicted sub-goal in \textit{VAST} plays the role of a value-backup anchor and an execution target.
    
However, this comparison is only the first step, far from the end of the story. We open the question of how to combine a learned dynamics (world) model into a better horizon-adaptive framework that enables a cleaner credit assignment. This question hinges on how to make a learned dynamic or world model, within the rapidly growing field of model-based learning, meaningfully shape what the policy should improve—so that an aggressive policy optimization process does not rely on imitation-like signals or a patchwork of separate components.

\subsection{Why condition $G$ on the Horizon?}
\label{appendix:why_horizon}

Since $\pi_\text{exec}$ does not consume the horizon $k$, one may ask whether $k$ could also be dropped from the $G$-function. However, horizon-free $G(s,s^+)$ admits infinitely many solutions, because a successor $s^+$ can be reached from $s$ at an arbitrary horizon. This leads to two problems: (1) increased estimation error in $G$, and (2) loss of the reachability information that $G$ is supposed to carry. Specifically, consider two trajectories that share the prefix $(s_0,a_0,r_0,s_1)$ and the same suffix leading to a successful terminal state $s_\text{success}$:
\begin{align*}
\tau_1 &= (s_0, a_0, r_0, s_1, a_1, r_1, s_2, a_2, r_2, s_3, a_3, r_3, s_4, \dots, s_\text{success}), \quad\text{with } s_1 = s_3 = s^+,\\
\tau_2 &= (s_0, a_0, r_0, s_1, a_1, r_1, s_2, \dots, s_\text{success}), \quad\text{with } s_1 = s^+ .
\end{align*}
Trajectory $\tau_2$ is the shorter and optimal, while $\tau_1$ contains a redundant detour from $s_1$ back to $s_3=s^+$. A horizon-free $G$ would be required to satisfy both
\begin{equation*}
V(s_0) \leftarrow G(s_0, s^+) + \gamma\, V(s^+) \qquad\text{and}\qquad V(s_0) \leftarrow G(s_0, s^+) + \gamma^{3}\, V(s^+).
\end{equation*}
A single scalar $G(s_0,s^+)$ cannot account for both the immediate and the delayed arrival at $s^+$. The rewards accumulated over different horizons are inseparably conflated, and $V(s_0)$ cannot be driven to its optimal value even when $V(s^+)$ is optimal. Moreover, the compositional property of Theorem~\ref{thm:compositional} is no longer well defined, thereby $G$ would also lose the temporal regularization.

\subsection{On the learning complexity of $G$}
\label{appendix:learning_complexity_of_g}   
With the definition of $G$ (Section~\ref{sec:grounded_return}), $G$ is in fact easier to fit than standard value functions. First, $G(s,s^+,k)$ conditions on richer and more controlled information, and regresses onto finite-horizon, grounded Monte-Carlo targets over at most $K$ steps (Eq.~(\ref{eql:g_mc})), rather than onto bootstrapped infinite-horizon discounted returns. Second, the compositional objective (Eq.~(\ref{eql:g_comp})) imposes a multi-scale consistency constraint that regularizes $G_\psi$ across horizons. As a result, the estimation error of $G$ is bounded by the finite-horizon target range and does not compound through recursive bootstrapping in the way that one-step TD errors do.

\section{Limitations and Future Works}
\label{appendix:limitation_future}

\subsection{Computation Cost Comparison}
\label{appendix:fullresults_computation}

The core components introduced by \textit{VAST} include an auxiliary value function and two policies with distinct functional roles. Consequently, with the same network architecture, \textit{VAST} has more parameters than the baselines. We report the parameter comparison in Table~\ref{table:param_count} and the computation time in Table~\ref{table:computation_time}. All methods are evaluated on the \texttt{cube-double-play} domain using an NVIDIA A100 GPU, with the layer hidden dimension fixed to $512$ and the batch size fixed to $512$. Parameter counts include the target networks, and computation time is measured over $1$M training steps. Under this controlled setting, the training time of \textit{VAST} is comparable to that of \textit{FLOQ}, remaining at a moderate computational cost. We consider this overhead a reasonable trade-off for the improved flexibility and optimality that \textit{VAST} offers.

\begin{table}[H]
\centering
\begin{minipage}{0.48\linewidth}
\centering
\captionof{table}{Total parameters.}
\label{table:param_count}
\begin{tabular}{lc}
\toprule
Methods & Parameter Count (M) \\
\midrule
VAST        & $6.6$  \\
FQL         & $4.9$  \\
QC          & $5.0$  \\
VALUE-FLOWS & $4.9$  \\
FLOQ        & $4.5$  \\
\bottomrule
\end{tabular}
\end{minipage}
\hfill
\begin{minipage}{0.48\linewidth}
\centering
\captionof{table}{Computation time per 1M steps.}
\label{table:computation_time}
\begin{tabular}{lc}
\toprule
Methods & Computation Time (H) \\
\midrule
VAST        & $1.7$  \\
FQL         & $0.6$  \\
QC          & $2.1$  \\
VALUE-FLOWS & $4.9$  \\
FLOQ        & $1.7$  \\
\bottomrule
\end{tabular}
\end{minipage}
\end{table}

\subsection{Scalability to High-Dimensional Observations}
\label{appendix:scalability}

In the current instantiation, $\pi_\text{stitch}$ predicts $s_k$ directly in the raw state space, and $G$ is conditioned on raw state pairs. All tasks considered in this work are state-based with moderate dimensionality (Table~\ref{table:ogbench_data}). For high-dimensional states or pixel-based observations, directly generating sub-goals in the observation space would increase the expressivity burden on $\pi_\text{stitch}$ and reduce training efficiency. However, \textit{VAST} can be readily combined with standard representation-learning techniques to maintain practical efficiency (e.g., an encoder-decoder architecture that maps raw observations into a low-dimensional latent space). Future work will make a validation on pixel-based setting.

\subsection{Greedier Diffusion Policy Optimization}
\label{appendix:diffusion_optimization}

In this work, we employ a diffusion policy extraction method that is more stable, albeit conservative. Meanwhile, recent studies have proposed more aggressive optimization procedures for diffusion policies, including~\citet{liang2026dipole} and~\citet{li2026qam}. Future work will investigate enhanced policy learning strategies to improve robustness and controllability.
\clearpage
\section*{NeurIPS Paper Checklist}

\begin{enumerate}

\item {\bf Claims}
    \item[] Question: Do the main claims made in the abstract and introduction accurately reflect the paper's contributions and scope?
    \item[] Answer: \answerYes{}
    \item[] Justification: Yes, the main claims made in the abstract and introduction accurately reflect the paper's contributions and scope.

\item {\bf Limitations}
    \item[] Question: Does the paper discuss the limitations of the work performed by the authors?
    \item[] Answer: \answerYes{}
    \item[] Justification: Yes, the paper discusses its limitations in Appendix~\ref{appendix:limitation_future}.

\item {\bf Theory assumptions and proofs}
    \item[] Question: For each theoretical result, does the paper provide the full set of assumptions and a complete (and correct) proof?
    \item[] Answer: \answerYes{}
    \item[] Justification: Yes, the paper provides the full set of assumptions and complete proofs for all theorems in Appendix~\ref{appendix:proof}.

    \item {\bf Experimental result reproducibility}
    \item[] Question: Does the paper fully disclose all the information needed to reproduce the main experimental results of the paper to the extent that it affects the main claims and/or conclusions of the paper (regardless of whether the code and data are provided or not)?
    \item[] Answer: \answerYes{}
    \item[] Justification: Yes, the paper fully discloses all information needed to reproduce the main experimental results. Pseudocode, experimental settings, hyperparameters, and baseline details are included in Appendix~\ref{appendix:experimental}-\ref{appendix:implementation}.

\item {\bf Open access to data and code}
    \item[] Question: Does the paper provide open access to the data and code, with sufficient instructions to faithfully reproduce the main experimental results, as described in supplemental material?
    \item[] Answer: \answerYes{}
    \item[] Justification: Yes, the code will be made available after the open-source approval process is completed.

\item {\bf Experimental setting/details}
    \item[] Question: Does the paper specify all the training and test details (e.g., data splits, hyperparameters, how they were chosen, type of optimizer) necessary to understand the results?
    \item[] Answer: \answerYes{}
    \item[] Justification: Yes, all hyperparameters and experimental setup details necessary to understand the results are provided in Appendix~\ref{appendix:experimental}-\ref{appendix:implementation}.

\item {\bf Experiment statistical significance}
    \item[] Question: Does the paper report error bars suitably and correctly defined or other appropriate information about the statistical significance of the experiments?
    \item[] Answer: \answerYes{}
    \item[] Justification: Yes, for all algorithms, several random seeds are used. All results are presented as mean ± standard deviation.

\item {\bf Experiments compute resources}
    \item[] Question: For each experiment, does the paper provide sufficient information on the computer resources (type of compute workers, memory, time of execution) needed to reproduce the experiments?
    \item[] Answer: \answerYes{}
    \item[] Justification: Yes, Appendix~\ref{appendix:limitation_future} provides detailed information about the specific computational demand used for the experiments.
    
\item {\bf Code of ethics}
    \item[] Question: Does the research conducted in the paper conform, in every respect, with the NeurIPS Code of Ethics \url{https://neurips.cc/public/EthicsGuidelines}?
    \item[] Answer: \answerYes{}
    \item[] Justification: We make sure to preserve anonymity.

\item {\bf Broader impacts}
    \item[] Question: Does the paper discuss both potential positive societal impacts and negative societal impacts of the work performed?
    \item[] Answer: \answerYes{}
    \item[] Justification: Yes, the main paper discusses both potential positive impacts and limitations of the work.
    
\item {\bf Safeguards}
    \item[] Question: Does the paper describe safeguards that have been put in place for responsible release of data or models that have a high risk for misuse (e.g., pre-trained language models, image generators, or scraped datasets)?
    \item[] Answer: \answerNA{}
    \item[] Justification: The paper poses no such risks.

\item {\bf Licenses for existing assets}
    \item[] Question: Are the creators or original owners of assets (e.g., code, data, models), used in the paper, properly credited and are the license and terms of use explicitly mentioned and properly respected?
    \item[] Answer: \answerYes{}
    \item[] Justification: Yes, all sources of data and code are properly credited in Appendix~\ref{appendix:experimental}, with licenses and terms of use clearly indicated.

\item {\bf New assets}
    \item[] Question: Are new assets introduced in the paper well documented and is the documentation provided alongside the assets?
    \item[] Answer: \answerNA{}
    \item[] Justification: The paper does not release new assets.

\item {\bf Crowdsourcing and research with human subjects}
    \item[] Question: For crowdsourcing experiments and research with human subjects, does the paper include the full text of instructions given to participants and screenshots, if applicable, as well as details about compensation (if any)? 
    \item[] Answer: \answerNA{}
    \item[] Justification: The paper does not involve crowdsourcing nor research with human subjects.

\item {\bf Institutional review board (IRB) approvals or equivalent for research with human subjects}
    \item[] Question: Does the paper describe potential risks incurred by study participants, whether such risks were disclosed to the subjects, and whether Institutional Review Board (IRB) approvals (or an equivalent approval/review based on the requirements of your country or institution) were obtained?
    \item[] Answer: \answerNA{}
    \item[] Justification: The paper does not involve crowdsourcing nor research with human subjects.

\item {\bf Declaration of LLM usage}
    \item[] Question: Does the paper describe the usage of LLMs if it is an important, original, or non-standard component of the core methods in this research? Note that if the LLM is used only for writing, editing, or formatting purposes and does \emph{not} impact the core methodology, scientific rigor, or originality of the research, declaration is not required.
    \item[] Answer: \answerYes{}
    \item[] Justification: Yes, see Appendix~\ref{appendix:llm}.

\end{enumerate}
\end{document}